\documentclass[10pt]{article}
\usepackage[margin=0.875in]{geometry}
\usepackage{times}
\usepackage[authoryear,round]{natbib}
\usepackage{microtype}
\usepackage{graphicx}
\usepackage{subcaption}
\usepackage{booktabs}
\usepackage[colorlinks=true,allcolors=blue!50!black]{hyperref}
\usepackage{amsmath}
\usepackage{amssymb}
\usepackage{mathtools}
\usepackage{amsthm}
\usepackage[capitalize,noabbrev]{cleveref}
\usepackage{multirow}
\usepackage{xcolor}
\usepackage{tikz}
\usetikzlibrary{arrows.meta,positioning,decorations.pathreplacing}
\usepackage{enumitem}

\theoremstyle{plain}
\newtheorem{proposition}{Proposition}
\theoremstyle{definition}
\newtheorem{definition}{Definition}
\theoremstyle{remark}

\newcommand{\iris}{\textsc{Iris}}
\newcommand{\jeva}{\textsc{Jeva}}
\newcommand{\onejev}{\textsc{OneJev}}
\newcommand{\enc}{\phi}
\newcommand{\dec}{\mathcal{D}}

\begin{document}

\title{Seeing the Options: What Should a Visual Candidate See?}
\author{\href{https://openreview.net/profile?id=~Fengxin_Li3}{Fengxin Li}\\Harbin Institute of Technology, Shenzhen\\\texttt{25b951129@stu.hit.edu.cn}}
\date{}
\maketitle

\begin{abstract}
Agents often choose among candidates that are images: elements of a screenshot, product photos, charts. A single-pass decision model either encodes the request and all candidates as one sequence, making its answer depend on candidate order, or encodes each candidate alone, which keeps order invariance and reusable encodings but hides the request from the candidate. We ask where the joint encoder's advantage comes from: from candidates seeing the request, or from candidates seeing each other while they are encoded. With one backbone and the same fine-tuning steps on screenshot crops of Multimodal-Mind2Web, the answer is mostly the request. Continuing each candidate from a forked cache of the encoded request recovers 11 to 13 of the 12 to 17 points that independent encoding loses to joint encoding, at the price of cross-request reuse, and we prove that order invariance stays exact whenever candidates never see each other. A joint encoder retrained with its candidates isolated from one another during encoding loses 1.1 and 0.4 ranking points in one paired run (0.7 and 0.3 against the mean of five joint runs), and permutation-equivariant interactions among candidates, over pooled vectors or inside the encoder, gain nothing beyond continued training. At a matched abstention operating point the state-conditioned model is 0.6 and 1.0 points below the joint encoder (mean of four runs) with a flip rate of zero, against 8--10\% for the joint encoder; on AndroidControl crops (three and five runs) it is 2.5 to 4 ranking points below. Finally, the abstention threshold drifts: the choice term of the loss ignores the score level, and the coverage term pins it only weakly once positives saturate; a quantile threshold aligns operating points.
\end{abstract}

\noindent Code, data and checkpoints: \url{https://github.com/lllfx/Iris}.

\section{Introduction}
\label{sec:intro}

Agents repeatedly make one-step choices among candidates: which element of a screenshot to click~\citep{mind2web,seeact}, which element of an app screen to tap~\citep{androidcontrol}, which retrieved image to keep~\citep{rerank2026}, or that none of them fits. Increasingly the candidates are images. A model for such choices should return a probability over the candidates and over ``none of them'' in one forward pass, should not change its answer when the candidates are listed in another order, and should not re-encode candidates that recur, as the elements of a page do at every step of a task that stays on it. Such single-pass models are called System One models~\citep{typesafe2026,jeva2026}.

The common way to build a decision model is joint encoding: the request and all candidates are serialised into one sequence, and the answer is read from the logits of answer letters. Large language models are known to change their multiple-choice answers when the options are reordered~\citep{pezeshkpour2024,pride}, and fine-tuning for decisions does not remove this dependence~\citep{runoff}. Multimodal System One models keep the candidates as text in letter slots~\citep{onejev2026}, and listing image candidates in such a prompt makes each request re-encode every candidate~\citep{rerank2026}. Inference-time remedies average over orders~\citep{pride} or compare a consideration set in every position~\citep{runoff}; nothing computed for a candidate can be reused.

At the other extreme, independent encoding, as in dual-encoder retrieval~\citep{dpr} and in \jeva{}~\citep{jeva2026}, computes each candidate's vector from the candidate alone and is order-invariant and cacheable by construction, but a candidate scored without the request cannot know which of two similar elements the instruction refers to.

\begin{figure}[t]
\centering
\includegraphics[width=3.25in]{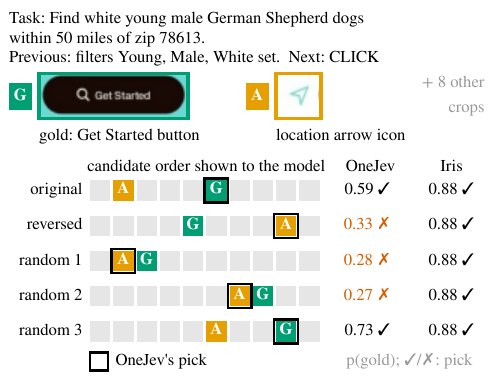}
\caption{A joint encoder's answer depends on where the candidates are listed; \iris{}'s does not. One cross-domain Mind2Web step with ten candidate crops: \onejev{}, fine-tuned on the same steps, picks the gold ``Get Started'' button (G) in two of five candidate orders and a location-arrow icon (A) in the other three. \iris{} (state-conditioned) gives the gold the same probability, 0.88, in every order. Numbers are each model's probability of the gold candidate.}
\label{fig:example}
\end{figure}

These designs differ in one respect: what a candidate's representation is allowed to depend on. It may depend on the candidate only, on the candidate and the request, or on the candidate, the request and the other candidates. Each step along this axis gives up a property. Seeing the request gives up reuse across requests; seeing the other candidates gives up order invariance unless the interaction is permutation-equivariant. The question this paper asks is where a joint encoder's accuracy comes from when the candidates are images: from candidates seeing the request, or from candidates seeing each other while they are encoded, and what each step costs in invariance and reuse.

We build \iris{}, a decision model on a vision--language backbone whose candidates may be images, text or both, and whose decision module is \jeva{}'s: the request state reads a memory of candidate vectors through attention, and a coverage gate assigns a probability to ``none of the candidates''. \iris{} implements two candidate encoders. The independent encoder computes each candidate's vector from the candidate alone. The state-conditioned encoder encodes the request once, keeps its attention and recurrent caches, and continues each candidate from a fork of that cache, so that each candidate sees the request and never another candidate.

Seeing the request is most of what matters. On a Mind2Web benchmark in which every candidate is the element's screenshot crop, letting the candidate see the request adds 11 to 13 points over independent encoding and brings the model to within about one point of a joint encoder fine-tuned on the same steps at a matched operating point (after checkpoint selection and an abstention stage the joint encoder lacks), with a flip rate of zero against 8--10\%; on AndroidControl the gap is 2.5 to 4 points in ranking accuracy. The comparison with independent encoding is controlled (same initialisation, data, steps and final weights); the one with the joint encoder is between systems. Retrained with its candidates isolated during encoding (its answer token still reads them all), the joint encoder loses 1.1 and 0.4 ranking points (0.7 and 0.3 against five joint runs; one isolated run), and permutation-equivariant interactions among our candidates gain no more than continued training. Only the independent encoder reuses candidates across requests.

A second finding concerns any gate that compares candidate scores with a threshold. The choice term of the loss ignores a common shift of the scores, and the coverage term pins their level only weakly: within one run, abstention on null steps swings between 27 and 94\% while ranking accuracy stays at 90--93\%. A quantile threshold aligns operating points, and we compare models at it throughout.

Our contributions are:
\begin{itemize}[leftmargin=1.2em,itemsep=1pt,topsep=2pt]
  \item A controlled account of where a joint encoder's accuracy over independent encoding comes from, along the axis of what a candidate may see. With one backbone and the same fine-tuning steps on visual-candidate benchmarks from Multimodal-Mind2Web and AndroidControl, seeing the request recovers 11--13 of 12--17 points on Mind2Web; candidates seeing each other during encoding is worth 0.3--1.1 ranking points, measured both by isolating the candidates of a retrained joint encoder and by adding equivariant interactions to ours (\cref{sec:bench,sec:main,sec:gap}).
  \item An order-invariant design that follows from it: a proof that order invariance holds exactly whenever candidates do not see each other and the decision module is symmetric, and a state-conditioned encoder that forks the cache of the encoded request, at one state pass per request and 1.3 times the joint latency; for the independent encoder, a closed-form cost model of candidate reuse within 6\% of measured latency for $K\ge64$ and partial reuse (\cref{sec:method,sec:reuse}).
  \item A diagnosis of the abstention gate: the choice term of the loss ignores the score level and the coverage term pins it only weakly, so the threshold drifts between checkpoints and seeds; a quantile threshold aligns operating points for comparing models (\cref{sec:null}, \cref{app:gate}).
\end{itemize}

% Three ways to encode K candidates: what each candidate's vector is allowed to see.
\begin{figure*}[t]
\centering
\begin{tikzpicture}[
  font=\footnotesize, >=stealth,
  tok/.style={draw, rounded corners=1.5pt, minimum height=5mm, minimum width=7mm, inner sep=2pt, align=center},
  st/.style={tok, fill=gray!15},
  cd/.style={tok, fill=green!12},
  vec/.style={draw, rounded corners=1.5pt, minimum height=5mm, inner sep=2pt, fill=orange!12, align=center},
  cache/.style={draw, dashed, rounded corners=2pt, fill=blue!6, inner sep=2pt, align=center},
  lab/.style={font=\footnotesize\itshape, gray!80!black, align=center},
  arr/.style={->, thick}
]
% ---------- (a) joint ----------
\begin{scope}[xshift=0cm]
  \node[st, minimum width=12mm] (s) at (0,0) {state};
  \node[cd] (c1) at (1.1,0) {$c_1$};
  \node[cd] (c2) at (1.85,0) {$c_2$};
  \node[lab] at (2.45,0) {$\cdots$};
  \node[cd] (ck) at (3.05,0) {$c_K$};
  \draw[decorate, decoration={brace, amplitude=3pt, mirror}] (-0.65,-0.35) -- (3.45,-0.35)
     node[midway, below=3pt, lab] {one sequence; letters A--Z read out};
  \node[vec] (o) at (1.4,1.3) {$p(\text{A}),\,p(\text{B}),\dots$};
  \draw[arr] (1.4,0.28) -- (o.south);
  \node[font=\bfseries\footnotesize, align=center] at (1.4,-1.7) {(a) joint (\onejev{})\\[-1pt] \normalfont every candidate\\[-1pt] \normalfont sees everything};
\end{scope}
% ---------- (b) independent ----------
\begin{scope}[xshift=4.55cm]
  \node[st, minimum width=12mm] (s) at (0,0) {state};
  \node[cd] (c1) at (1.25,0) {$c_1$};
  \node[cd] (c2) at (2.0,0) {$c_2$};
  \node[lab] at (2.6,0) {$\cdots$};
  \node[cd] (ck) at (3.2,0) {$c_K$};
  \node[vec] (z) at (0,1.2) {$z_0$};
  \node[vec] (u1) at (1.25,1.2) {$u_1$};
  \node[vec] (u2) at (2.0,1.2) {$u_2$};
  \node[vec] (uk) at (3.2,1.2) {$u_K$};
  \draw[arr] (s) -- (z); \draw[arr] (c1) -- (u1); \draw[arr] (c2) -- (u2); \draw[arr] (ck) -- (uk);
  \node[cache, minimum width=26mm] at (2.2,2.2) {option memory:\\reused across requests};
  \draw[gray, dashed] (2.2,1.5) -- (2.2,1.75);
  \node[font=\bfseries\footnotesize, align=center] at (1.6,-1.7) {(b) independent (\jeva{}, \iris{})\\[-1pt] \normalfont a candidate\\[-1pt] \normalfont sees only itself};
\end{scope}
% ---------- (c) state-conditioned ----------
\begin{scope}[xshift=9.25cm]
  \node[st, minimum width=12mm] (s) at (0,0) {state};
  \node[cache, text width=15mm] (kv) at (0,1.2) {state cache\\(computed\\once)};
  \draw[arr] (s) -- (kv);
  \node[cd] (c1) at (2.0,0.35) {$c_1$};
  \node[cd] (c2) at (2.0,1.2) {$c_2$};
  \node[lab] at (2.0,1.75) {$\vdots$};
  \node[cd] (ck) at (2.0,2.3) {$c_K$};
  \node[vec] (u1) at (2.95,0.35) {$u_1$};
  \node[vec] (u2) at (2.95,1.2) {$u_2$};
  \node[vec] (uk) at (2.95,2.3) {$u_K$};
  \draw[arr] (kv.east) -- (c1.west); \draw[arr] (kv.east) -- (c2.west); \draw[arr] (kv.east) -- (ck.west);
  \draw[arr] (c1) -- (u1); \draw[arr] (c2) -- (u2); \draw[arr] (ck) -- (uk);
  \node[lab] at (1.4,2.95) {forked, never joined};
  \node[font=\bfseries\footnotesize, align=center] at (1.35,-1.7) {(c) state-conditioned (\iris{})\\[-1pt] \normalfont a candidate sees\\[-1pt] \normalfont the state only};
\end{scope}
% ---------- (d) candidate-side readouts ----------
\begin{scope}[xshift=13.6cm]
  \node[vec] (u1) at (0,0.35) {$u_1$};
  \node[vec] (u2) at (0,1.2) {$u_2$};
  \node[lab] at (0,1.75) {$\vdots$};
  \node[vec] (uk) at (0,2.3) {$u_K$};
  \node[draw, rounded corners=2pt, fill=orange!6, minimum height=24mm, minimum width=14mm, align=center, font=\footnotesize] (mix) at (1.45,1.3) {set\\attention /\\multi-vector\\(\cref{sec:gap})};
  \draw[arr] (u1.east) -- (mix.west |- u1); \draw[arr] (u2.east) -- (mix.west |- u2); \draw[arr] (uk.east) -- (mix.west |- uk);
  \node[lab] at (1.0,2.95) {equivariant, no gain};
  \node[font=\bfseries\footnotesize, align=center] at (0.9,-1.7) {(d) readouts over $\{u_i\}$\\[-1pt] \normalfont pooled candidates\\[-1pt] \normalfont see each other};
\end{scope}
\end{tikzpicture}
\caption{The designs differ only in what a candidate's vector may depend on. (a) Joint encoding: the answer depends on the order and number of candidates and nothing is reusable. (b) Independent encoding: $u_i$ is a function of $c_i$ alone, so the decision is order-invariant, unbounded in $K$ and cacheable across requests (\cref{prop:inv}). (c) State-conditioned encoding: each candidate continues from a fork of the state's cache; order invariance and null handling are unchanged, the state pass is shared within a request, and cross-request reuse is given up. (d) Readouts that let the pooled candidate vectors see each other keep every guarantee of (c) and, in our experiments, add no more than continued training (\cref{sec:gap}).}
\label{fig:designs}
\end{figure*}
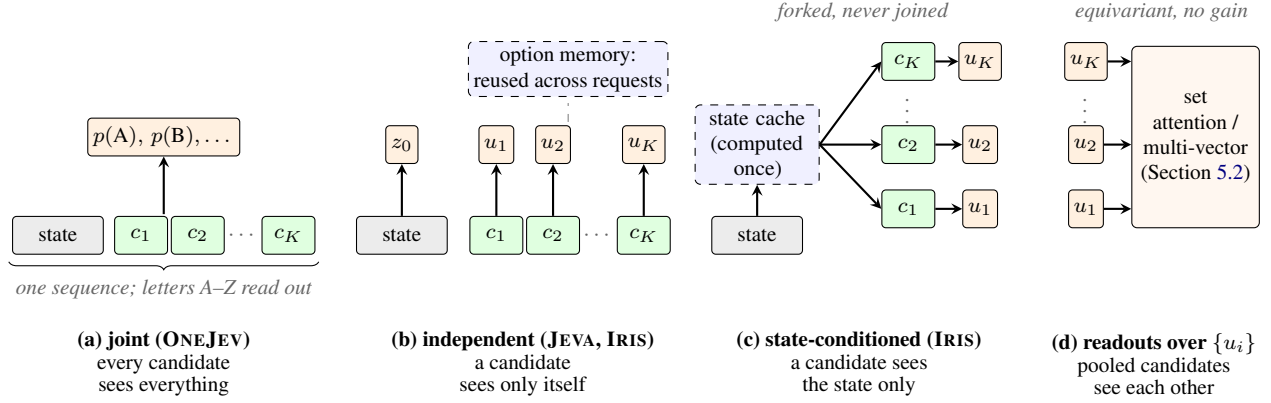

\section{Problem and Design Axis}
\label{sec:problem}

A request carries a state $x$, a question $\iota$, a decision type $t$ (\textsc{choice}; \textsc{noul}, a binary yes/no judgement; or an ordinal \textsc{score}) and a candidate set $C=\{c_1,\dots,c_K\}$. For \textsc{choice} the model returns a distribution over $C\cup\{\bot\}$, where $\bot$ means that no candidate applies. A step is a \emph{null step} if its gold is $\bot$ and a \emph{step with a positive} if its candidate set contains the gold.

\begin{definition}[Segment]
A segment is a pair $(s,\mathbf{I})$ of a text string $s$ and an ordered tuple of images $\mathbf{I}=(I_1,\dots,I_m)$, $m\ge 0$, each referenced once in $s$ by a placeholder. A visual-candidate request is one whose state and candidates are segments.
\end{definition}

A screenshot crop is a segment with one image and a short label; a text candidate is the case $m=0$. Following \jeva{}, we ask of a decision model: (R1) an explicit probability $p(\bot)$; (R2) outputs invariant to the order of $C$; (R3) cost linear in $K$ with no fixed upper bound; (R4) candidate computation that can be cached and reused across requests.

\paragraph{The axis.} Let $u_i\in\mathbb{R}^d$ denote the vector the decision module reads for candidate $c_i$. Four designs differ in the arguments of $u_i$ (\cref{fig:designs}, \cref{tab:designs}). In \emph{joint} encoding the candidates are listed in one sequence after the state and the answer is read from letter logits; every candidate's representation depends on the state, on the candidates listed before it and on its position, and the answer reads all of them. In \emph{independent} encoding $u_i=\enc(c_i,t)$. In \emph{state-conditioned} encoding $u_i=\enc(x,\iota,t,c_i)$: the candidate sees the request but no other candidate. Finally, a permutation-equivariant readout over $\{u_i\}$ can let pooled candidate vectors see each other after encoding.

\begin{table}[t]
\centering\small
\setlength{\tabcolsep}{3pt}
\caption{Designs against R2 (order), R4 (reuse) and R1 (null), and candidate encodings per request; $m$ orders averaged, $r$ fraction cached. ``Letter'': a ``none'' option in the letter slots, without a coverage gate.}
\label{tab:designs}
\setlength{\tabcolsep}{3pt}
\begin{tabular}{lcccc}
\toprule
Design & R2 & R4 & R1 & Encodings \\
\midrule
Joint (letter slots) & no & no & letter & $K$ jointly \\
Joint, order-averaged & approx. & no & letter & $mK$ jointly \\
Independent & exact & across requests & yes & $(1-r)K$ \\
State-conditioned & exact & within request & yes & $1+K$ \\
\quad + readout & exact & within request & yes & $1+K$ \\
\bottomrule
\end{tabular}
\end{table}

\section{Method}
\label{sec:method}

\iris{} keeps the full Qwen3.5-4B vision--language model~\citep{qwen35}, which \jeva{} loads and reduces to its language model. A segment is rendered into one token sequence in which each image placeholder becomes the backbone's vision block; the language model reads the mixed sequence, and the last token is pooled and projected to $d=768$. The vision tower is frozen; the language model carries a rank-64 LoRA~\citep{lora}. The decision module and the training objective are \jeva{}'s and are not modified.

\subsection{Decision module}
\label{sec:decision}
Each of $N$ decision blocks lets the current state attend over the candidate memory $\{u_i\}$ (keys $\ell_2$-normalised), adds the message to the state through a learned gate, and applies a feed-forward layer. Only the state is updated; no $K\times K$ interaction is formed. The scores $s_i=\langle z_N,\hat u_i\rangle$ give $q=\mathrm{softmax}(s/\tau_{\text{dec}})$, and a coverage gate $\rho=\sigma\big((\max_i s_i-b)/\tau_{\text{cov}}\big)$ gives $p(\bot)=1-\rho$ and $p_i=\rho\,q_i$. We use $N=1$ with 8 heads. Training minimises the joint cross-entropy over $\{\bot\}\cup C$.

\subsection{Invariance when candidates do not see each other}
\label{sec:guarantees}

\begin{proposition}[Invariance and reuse]
\label{prop:inv}
Let $u_i=\enc(x,\iota,t,c_i)$ for any function $\enc$ that does not take the other candidates as input, and let $\dec$ be the decision stack above with fixed weights, evaluated exactly. (a) For any permutation $\pi$, the state trajectory, the coverage $\rho$ and the \textup{\textsc{noul}} and \textup{\textsc{score}} outputs are unchanged under $C\mapsto\pi C$, and $q$ is permuted accordingly. (b) If in addition $\enc$ does not depend on $(x,\iota)$, replacing any $u_i$ and its per-layer key and value projections by entries computed in an earlier request from the same segment leaves every output unchanged.
\end{proposition}
\begin{proof}
(a) Given $(x,\iota,t)$, each $u_i$ is a function of $c_i$ alone, so permuting $C$ permutes $\{u_i\}$. Each block computes a softmax over $i$ and a weighted sum of values, both invariant to relabelling; the maximum in the coverage gate and the softmax over $s_i$ are likewise invariant, and $q$ follows the relabelling. (b) $u_i$ and its projections are deterministic functions of $c_i$; an entry computed earlier from an identical segment is identical.
\end{proof}

Part (b) is \jeva{}'s argument; the independent encoder satisfies both parts. Part (a) is what makes the state-conditioned encoder admissible: it gives up cross-request reuse and keeps exact order invariance, the unbounded candidate count, and the coverage gate. In bf16 the outputs under reordering agree to within $10^{-7}$; the only order-dependent argmax we observed is a tie between candidates with identical pixels and text.

\paragraph{Cost of reuse.} Let $a$ be the cost of encoding the state, $c$ the cost of encoding one candidate, $\delta\ll c$ the per-candidate cost of the decision read, and $r$ the fraction of candidates found in the cache.

\begin{proposition}[Latency under reuse]
\label{prop:cost}
For the independent encoder, per-request cost is $T(r)=a+(1-r)Kc+K\delta$, and the speedup at full reuse is
\[
S(1)=\frac{a+Kc+K\delta}{a+K\delta}=1+\frac{Kc}{a+K\delta}.
\]
\end{proposition}
The limit grows with $c$, so reuse is worth more for image candidates than for text candidates, by the ratio of their encoding costs. A cache entry is keyed by the rendered text and a digest of each image's bytes, so identical candidates share an entry regardless of where they came from.

\subsection{State-conditioned encoding with a forked cache}
\label{sec:cond-method}
Encoding each candidate after the state naively costs $K$ passes over the state. We encode the state once with its key/value caches retained (Qwen3.5 mixes full attention with gated linear attention, and both kinds of layer keep a cache), and continue each candidate from a fork of that cache. The fork is an index selection, so during training the gradients of all $K$ candidates flow back into the single state pass. The state is left-padded so that every row's cache ends at the fork point, and positions continue from the state's last position; the result equals the concatenated sequence $[\text{state};\text{candidate}_i]$ up to bf16 rounding, which we verified against a concatenation implementation. Per request the cost is one state pass plus $K$ candidate continuations, and the state's cache is shared by every question asked about the same state.

\subsection{Training}
\label{sec:training}
Both encoders start from \jeva{}-4B's weights, remapped to the vision--language module layout; on text inputs the remapped model reproduces \jeva{}'s outputs exactly (maximum absolute difference 0). Stage 1 trains on OneJev-Data~\citep{onejev2026}, the released 94{,}707 of the 99{,}193 questions on which \onejev{}-4B was trained, in which images appear in the state and candidates are text (we use the checkpoint after 80k rows; \cref{app:zeroshot} reports it and the untrained transfer). Stage 2 trains on visual-candidate steps for two epochs. For the state-conditioned model we add a short abstention stage: one epoch at a fifth of the learning rate on a subset in which every null step appears three times and half of the steps with a positive are kept (24.7\% null). Hyper-parameters are in \cref{app:hparams}.

\section{Visual-Candidate Benchmarks}
\label{sec:bench}

\paragraph{Mind2Web crops.} Multimodal-Mind2Web~\citep{mind2web,seeact} records tasks on real websites; at each step the agent selects one element, and the multimodal release adds the page screenshot and every element's bounding box. We keep \jeva{}'s three test splits (cross-task, cross-website, cross-domain), its $K=16$ and its positives, and replace each candidate by the element's crop (6\,px padding, minimum side 48\,px, longest side 336\,px) with a short label of the tag and salient attributes (\cref{fig:example}).

Two details make order invariance measurable. Pages often contain elements with identical appearance and markup, whose argmax is decided by position; we deduplicate candidates by text and image bytes, keeping the positive. Our construction places the positive first, so we shuffle each record's candidates with a seed derived from its id; otherwise a joint encoder could learn that the answer is the first slot. The result has 7{,}751 / 1{,}337 / 1{,}017 / 4{,}042 steps for train / cross-task / cross-website / cross-domain. Steps whose gold element is absent from the cleaned HTML (4.3--6.1\% per split) are kept as null targets. These accuracies are not comparable with GUI-agent element accuracies on standard Mind2Web, which rank a page's released candidate elements (about 550 per cross-task step), usually from HTML; here a step has at most 16 crops and null steps count.

\paragraph{AndroidControl crops.} AndroidControl~\citep{androidcontrol} records demonstrations on Android apps with an accessibility tree and a screenshot per step, and defines test sub-splits for in-distribution episodes and for unseen apps, tasks and app categories. We convert every click and long-press step into a \textsc{choice}: the state is the episode goal and the instructions of the previous steps, and the candidates are crops of the screen's visible, enabled, clickable elements, clipped to the screen and subject to size limits, with the same crop and label rules as above. The gold candidate is the smallest clickable element that contains the touch point (rules for targets not marked clickable are in \cref{app:data}); steps without a gold (1.2\% of training and 1.5\% of test steps) are kept as null targets. Deduplication, shuffling and $K=16$ follow the Mind2Web construction. The build has 46{,}079 training and 4{,}990 test steps (2{,}421 in-distribution, 1{,}965 unseen-app, 2{,}523 unseen-task and 2{,}183 unseen-category; the sub-splits overlap).

\section{Experiments}
\label{sec:exp}

\paragraph{Setup.} All models use Qwen3.5-4B. The joint baseline is \onejev{}-4B~\citep{onejev2026}, which shares the backbone and, up to the 4{,}486 questions whose sources do not permit redistribution, the stage-1 data. We use the first release, the only one with public training data. Because it cannot take image candidates, we list every candidate's crop inside its state and keep letter slots for the candidates. We fine-tune it with LoRA on the same 7{,}751 Mind2Web steps in that format. The two sides are not tuned identically: \onejev{} starts from its own release, uses LoRA rank 32, reads crops with 25\% more image tokens (upscaled; this does not help \iris{}, \cref{app:gap}) and is reported at its final weights, whereas \iris{} starts from \jeva{}-4B plus stage 1, uses rank 64, and its released model adds checkpoint selection, the abstention stage and a margin, all chosen on cross-task (\cref{app:hparams}).

Each record is evaluated in its original order, reversed, and three seeded random orders. The \emph{flip rate} is the fraction of records whose argmax differs across the five orders; a change within a set of candidates whose probabilities agree to $10^{-6}$ is a numerical tie and is not counted. Accuracy is first-choice accuracy in the original order and counts null steps. Caches are cleared before every request unless reuse is measured.

\subsection{What seeing the request buys}
\label{sec:main}

\begin{table}[t]
\centering\small
\setlength{\tabcolsep}{5pt}
\caption{Element selection on Mind2Web crops ($K{=}16$): first-choice accuracy and flip rate over five candidate orders. All rows except \jeva{} use Qwen3.5-4B and the same crops; the two \iris{} stage-2 rows differ only in what a candidate sees (controlled), the comparison with \onejev{} is between systems; all but the first three are fine-tuned on the Mind2Web crop steps, the abstention rows then on null-oversampled steps. $^\dagger$Selection split: the checkpoint and the abstention margin of the released model were chosen on cross-task. Zero-shot: letter logits of the unmodified chat model (\cref{app:zeroshotchat}). The averaging row averages the joint encoder over the five orders (five passes per request); its flip rate is not measured. The set-attention rows continue training from the released conditioned model and from the mixture-trained independent model (\cref{app:mixture}), so they are compared with those in \cref{sec:gap}; the conditioned one has the same cross-website accuracy as the released model at $\delta=0$ (86.53\%), although 110 of its 1{,}017 decisions differ.}
\label{tab:main}
\begin{tabular}{ll ccc ccc}
\toprule
& & \multicolumn{3}{c}{Accuracy (\%)} & \multicolumn{3}{c}{Flip rate (\%)} \\
\cmidrule(lr){3-5}\cmidrule(lr){6-8}
Model & Candidate sees & Task & Website & Domain & Task & Website & Domain \\
\midrule
\jeva{}-4B, HTML text (reported) & itself & 79.99 & 62.25 & 66.23 & 0.00 & 0.00 & 0.00 \\
\onejev{}-4B, released & everything & 37.85 & 37.46 & 34.83 & 50.64 & 52.80 & 54.30 \\
Qwen3.5-4B chat, zero-shot letters & everything & 26.40 & 24.68 & 29.22 & 78.38 & 79.45 & 78.13 \\
\onejev{}-4B, fine-tuned (joint) & everything & 90.58 & \textbf{89.48} & 91.59 & 7.63 & 10.13 & 7.99 \\
\quad + averaging over 5 orders & everything & 91.77 & 89.28 & \textbf{92.11} & -- & -- & -- \\
\midrule
\iris{} independent & itself & 78.38 & 72.07 & 74.84 & 0.00 & 0.00 & 0.00 \\
\iris{} state-conditioned & the request & 89.60 & 84.46 & 87.48 & 0.00 & 0.00 & 0.00 \\
\quad + abstention stage and margin & the request & \textbf{92.07}$^\dagger$ & 87.02 & 90.20 & 0.00 & 0.00 & 0.00 \\
\quad + abstention stage, 3 seeds ($\delta=0$) & the request & 91.90$^\dagger$ & 87.45 & 90.42 & 0.00 & 0.00 & 0.00 \\
\quad \quad sd over seeds & & 0.28 & 0.79 & 0.56 & & & \\
\midrule
\iris{} independent + set attention & itself, then set & 80.55 & 73.55 & 76.92 & 0.00 & 0.00 & 0.00 \\
\iris{} conditioned + set attention & request, then set & 89.75 & 86.53 & 89.91 & 0.00 & 0.00 & 0.00 \\
\bottomrule
\end{tabular}
\end{table}

\Cref{tab:main} gives the main comparison. The released \onejev{} can read the crops but not stably: it changes its answer on 51--54\% of records when the same crops are listed in a different order. Fine-tuning on the 7{,}751 crop steps raises its accuracy to 90.58 / 89.48 / 91.59\% and lowers its flip rate to 7.6--10.1\%, but does not remove it; its all-order accuracy, which credits a record only if all five orders are answered correctly, is 84.4\% on cross-website and 87.8\% on cross-domain. Averaging its distribution over the five orders reduces the order dependence without removing it (\cref{sec:reuse}) and reaches 91.77 / 89.28 / 92.11\%, at the cost of five forward passes per request, none of whose candidate computation can be reused; its abstention on null steps falls to 56--68\%.

Independent encoding, trained on the same steps from the same stage-1 checkpoint, reaches 78.38 / 72.07 / 74.84\% with a flip rate of exactly zero on all 6{,}396 test records. It is 9.8 and 8.6 points above \jeva{}'s HTML-text candidates on cross-website and cross-domain (1.6 below on cross-task), a reference point rather than a controlled comparison.

The state-conditioned encoder is the controlled test of what independent encoding loses: it adds the request to what a candidate sees and withholds the other candidates. Trained with the same steps and epochs and reported at its final checkpoint, it reaches 89.60 / 84.46 / 87.48\%, which is 11.2 / 12.4 / 12.6 points above independent encoding, with a flip rate of zero. Most of the gap to joint encoding is therefore the candidate not seeing the request. It relies mainly on the crops: retrained without labels it ranks 85--87\% of cross-website and cross-domain positives correctly, without crops 53--56\% (\cref{app:modality}).

After the abstention stage and a margin $\delta=0.24$ on the null decision chosen on cross-task (\cref{app:margin}), the model reaches 92.07 / 87.02 / 90.20\%: 2.5 points below the fine-tuned joint encoder on cross-website and 1.4 on cross-domain, and above its all-order accuracy on both. At $\delta=0$ the recipe reaches 87.45 $\pm$ 0.79 / 90.42 $\pm$ 0.56 over three seeds; the released model's 86.53\% is its low end (\cref{tab:main}).

At a matched operating point, with every model's threshold set to the same quantile of its own cross-task scores (the same abstention budget, 9.4\% of cross-task requests; \cref{app:gate}), the state-conditioned recipe (mean of three seeds) is 0.95 and 0.87 points below the joint encoder (one run; \cref{tab:matched}), with task-level 95\% intervals of $[-3.01,+1.14]$ and $[-1.77,-0.02]$ (\cref{app:cluster}). Against four more joint-encoder runs it is 0.61 and 1.00 points below ($[-2.29,+1.15]$, $[-1.73,-0.30]$). Without checkpoint selection and the abstention stage, which the joint encoder does not get, its stage-2 final weights are 3.2 and 2.9 points below; skipping only the selection costs about half a point (\cref{app:gap}).

\begin{table}[t]
\centering\small
\setlength{\tabcolsep}{2.5pt}
\caption{Mind2Web crops at a matched abstention operating point: each model abstains when $\operatorname{logit} p(\bot)$ exceeds the 0.906-quantile of its own cross-task requests (\cref{app:gate}). Accuracy on all steps / steps with a positive / abstention on null steps (\%). The state-conditioned rows average three seeds; the first is the released recipe (checkpoint selection and abstention stage), the second its stage-2 final weights. The joint encoder is one run at its final weights; the row below it averages four more runs with the same recipe (seeds 7, 13, 21, 42). Gaps in the text use unrounded values (cross-domain $90.203-89.337=0.87$).}
\label{tab:matched}
\setlength{\tabcolsep}{3pt}
\begin{tabular}{lcc}
\toprule
Model & Website & Domain \\
\midrule
\onejev{} fine-tuned & 86.04\,/\,86.13\,/\,84.1 & 90.20\,/\,90.79\,/\,80.2 \\
\quad four more runs & 85.69\,/\,85.79\,/\,83.5 & 90.34\,/\,90.83\,/\,81.9 \\
\iris{} conditioned & 85.09\,/\,85.27\,/\,81.0 & 89.34\,/\,89.79\,/\,81.5 \\
\quad stage-2 final & 82.86\,/\,83.04\,/\,78.8 & 87.30\,/\,87.94\,/\,76.3 \\
\iris{} independent & 70.89\,/\,72.25\,/\,40.9 & 71.67\,/\,73.72\,/\,36.5 \\
\bottomrule
\end{tabular}
\end{table}

\paragraph{AndroidControl.} \Cref{tab:ac} lists the same comparison on the AndroidControl sub-splits. The state-conditioned model is trained for one epoch on the 46{,}079 training steps from the same stage-1 checkpoint and reported at its final weights (three seeds, against five runs of \onejev{}); it reaches 79.22 / 75.42 / 76.32 / 75.49\% with a flip rate of zero. Natural null steps are rare (1.5\% of test steps), so null detection uses test copies with the gold removed from a seeded 10\% of steps (AUROC 0.82--0.83). An abstention stage like Mind2Web's raises abstention on them from 11--12\% to 64--69\% but costs 1.7--6.7 points on the unmodified splits (seed 42), so the table omits it.

As on Mind2Web, the state-conditioned model is below the fine-tuned \onejev{} in the original order, by 2.8 points in distribution and 3.9--4.3 on the unseen sub-splits (accuracy intervals exclude zero; \cref{app:acmatched}), by 2.6--4.9 at a matched operating point, with a null AUROC 0.9--2.2 points lower; it is 1.7--2.1 points above \onejev{}'s all-order accuracy, since \onejev{} changes its answer on 13--16\% of steps. Seeing the request recovers about half of the gap between independent and joint encoding here (3.3--5.8 of 6.1--10.1 points, against 71--92\% on Mind2Web), and on the unseen sub-splits the conditioned model varies more across seeds (sd 0.9--1.5 points, \onejev{} 0.1--0.2).

\begin{table}[t]
\centering\small
\setlength{\tabcolsep}{3pt}
\caption{Element selection on AndroidControl crops: accuracy (\%) on the test sub-splits (IDD: in-distribution; unseen app, task and category). All models are reported at their final weights: \onejev{} is the mean of five runs (sd 0.08--0.18), \iris{} state-conditioned of three seeds (sd 0.23 in distribution, 0.91--1.51 unseen), independent one run. \iris{}'s flip rate is zero on every sub-split.}
\label{tab:ac}
\begin{tabular}{lcccc}
\toprule
Model & IDD & App & Task & Category \\
\midrule
\onejev{}-4B fine-tuned (joint) & 82.06 & 79.31 & 80.31 & 79.81 \\
\quad all five orders correct & 77.15 & 73.56 & 74.59 & 73.78 \\
\quad averaged over 5 orders & 83.23 & 80.11 & 81.15 & 80.73 \\
\iris{} independent & 75.96 & 70.64 & 71.22 & 69.72 \\
\iris{} state-conditioned & 79.22 & 75.42 & 76.32 & 75.49 \\
\bottomrule
\end{tabular}
\end{table}

\subsection{What letting candidates see each other buys}
\label{sec:gap}

If the joint encoder's lead came from comparing candidates while they are encoded, isolating them should cost it that lead, and letting ours interact equivariantly, which keeps the decision order-invariant, should recover it; neither happens. Retraining \onejev{} with its candidates isolated in every layer (one run) costs it 1.13 / 0.39 of its 3.4 / 2.2-point ranking lead over stage-2 \iris{} (three seeds), or 0.72 / 0.28 against the mean of five joint runs (website / domain; \cref{app:isolation}); its answer token still reads every candidate, so comparison at the readout is not removed. Refitting \iris{}'s head on the frozen encoder recovers 1.2 / 0.7 more, mostly on one seed (\cref{app:gap}).

On our side, we tried readouts over the pooled vectors (a set-attention block~\citep{settransformer}, a token-level read of each candidate's token states, a multi-vector memory) and an interaction inside the encoder that starts each candidate's Gated DeltaNet state~\citep{deltanet} from the other candidates' writes, weighted as an average over their orders~\citep{picaso}. Each continues training a released model for one epoch, so we compare it with one more epoch without any module, on ranking accuracy. None moves beyond that control: set attention changes accuracy by $-0.5$ / $+0.6$ / $+0.9$ points against $-1.4$ / $+2.0$ / $+0.2$ for the control, and the in-encoder interaction by $-0.22$ after head refit ($[-0.73,+0.30]$; \cref{app:encinter}). The remaining errors sit on the quarter of steps whose gold and hardest negative vectors nearly coincide (ranking about 70\% there, 99\% elsewhere), although their crops differ visibly; no interaction we tried, and no higher image resolution, helps there (\cref{app:gap}). These negative results cannot exclude an interaction that other training would find (the starting model ranks 97.4\% of training steps correctly).

\paragraph{Geometry of the candidate vectors.}\label{app:geometry} The set-attention block learned almost nothing: its output projections stayed near zero, it changed the candidate vectors by 0.3\% in norm, and the normalised entropy of its attention rows was 0.975, where 1 is uniform.
The candidate vectors are close to each other. On cross-website, the cosine between the gold candidate's vector and its hardest negative averages 0.984 on steps the model answers correctly and 0.998 on steps it answers wrongly, and on every wrong step the chosen candidate is within cosine 0.986 of the gold. The crops are not alike: on wrong records the median pixel RMSE between the gold and the chosen crop is 0.165, and of 65 errors 4 are pixel twins and 20 share the label text. The distinguishing information is also not recoverable from the token states by similarity: an untrained MaxSim score \citep{colbert} between the state's and each candidate's token states prefers the gold over the chosen candidate on 5 of 19 wrong records.

In the multi-vector memory, 86--99\% of the energy of the extra vectors lies in a direction shared by all candidates. The candidate vectors of one request thus occupy a narrow cone, consistent with the anisotropy reported for contextual representations~\citep{ethayarajh2019}, and the decision module separates them through its temperature. Whether pooling has discarded what a readout would need is not settled by this statistic: since the vectors share one dominant direction, their raw cosines are close to one whatever the orthogonal components carry, and we do not test scoring candidates without pooling.

\subsection{Null detection}
\label{sec:null}
Null steps (81 / 44 / 222 per split) test R1. \Cref{tab:null} (\cref{app:tradeoff}) gives the AUROC of $p(\bot)$ for separating them from steps with a positive, and the abstention rate, the share of null steps on which $\bot$ is the answer. The state-conditioned model ranks null steps well after stage 2 (AUROC 0.92--0.95) but abstains on only 34--43\% of them; the abstention stage raises abstention to 77--79\% and the AUROC to 0.93--0.97. At a matched operating point and over three seeds, the stage adds 1.6 and 0.9 points, but one more epoch at the same learning rate without null oversampling adds 1.1 and 0.5 (\cref{app:abststage}): most of the gain is the extra epoch. The independent model, whose candidates never see the instruction, is weaker on both.

\paragraph{The gate's threshold drifts.} Across the stage-2 checkpoints of one run (seed 7; \cref{tab:gate-symptom}), abstention on the cross-task null steps swings between 27 and 94\% while ranking accuracy stays at 90--93\%: once positive steps saturate their coverage, mainly the few null steps constrain the score level, and the gate compares $\max_i s_i$ with an almost fixed $b$ (\cref{app:gate}). Abstaining instead when $\operatorname{logit} p(\bot)$ exceeds the $q$-quantile of unlabelled requests ($q$ chosen once on cross-task) aligns operating points without changing AUROC or calibrating probabilities, and cuts the spread of abstention across 25 checkpoints from 16.5 to 4.9 points.

\subsection{Scaling and reusing image candidates}
\label{sec:reuse}

We measure R3 and R4 with the mixture-trained independent model (\cref{app:mixture}) on one A800 GPU, on cross-task steps whose candidate pools are grown to $K\in\{16,64,256,1024,2048\}$ crops (the gold is always present; we measure cost, not accuracy, at larger $K$). The state costs $a=52.8$\,ms; one crop (69--75 tokens) costs 8.7--11.6\,ms and its text label 3.6\,ms, a ratio of 2.4--3.2. With caches cleared, latency grows linearly with $K$, from 208\,ms at $K=16$ to 23.9\,s at $K=2{,}048$, and candidates are encoded in chunks, so peak activation memory does not grow with $K$.

\Cref{tab:reuse} compares measured median latency with $T(r)$ when exactly $\lfloor rK\rceil$ of the crops are already cached. For $K\ge64$, measured latency is within 6\% of $T(r)$ at $r\in\{0.25,0.5,0.75\}$. The model departs at $K=16$, where per-request overhead it omits matters (+16 to +50\%), and at $r=1$, where assembling cached entries adds 3 to 71\,ms. The speedup at full reuse grows with $K$ because $a$ is paid once: $3.6\times$ at $K=16$, $37\times$ at 256 and $191\times$ at 2{,}048.

\begin{table}[t]
\centering\small
\setlength{\tabcolsep}{3pt}
\caption{Median latency (ms), measured / predicted by \cref{prop:cost}, with $K$ image candidates of which a fraction $r$ is cached; fresh state per request, one A800, 5 repetitions. The constants are measured per $K$ on the same crops, so $r{=}0$ is not a test of the model; $S(1)$ is measured.}
\label{tab:reuse}
\begin{tabular}{rrrrr}
\toprule
$K$ & $r{=}0$ & $r{=}0.5$ & $r{=}1$ & $S(1)$ \\
\midrule
16 & 208 / 201 & 147 / 127 & 57 / 54 & 3.6$\times$ \\
64 & 652 / 628 & 339 / 341 & 58 / 54 & 11$\times$ \\
256 & 2{,}342 / 2{,}275 & 1{,}208 / 1{,}164 & 64 / 54 & 37$\times$ \\
1{,}024 & 12{,}108 / 11{,}907 & 6{,}083 / 5{,}980 & 91 / 54 & 133$\times$ \\
2{,}048 & 23{,}894 / 23{,}894 & 11{,}909 / 11{,}974 & 125 / 54 & 191$\times$ \\
\bottomrule
\end{tabular}
\end{table}

On the cross-task split with full-page candidate pools, 45\% of the 658{,}639 candidates an agent is shown are crops (identical image bytes and label) it has already been shown in the same task, and the median step after a task's first shares 61\% of its candidates with earlier steps of that task. The state-conditioned encoder cannot exploit this cross-request reuse; it shares only the state pass within a request.

\begin{figure}[t]
\centering
\includegraphics[width=3.25in]{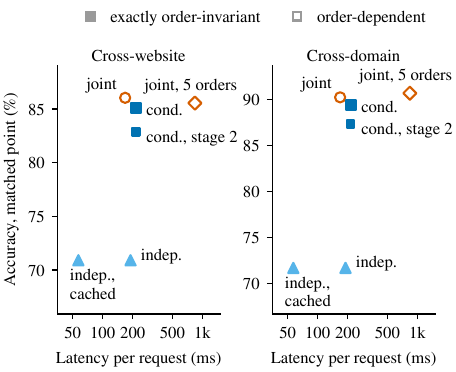}
\caption{The conditioned encoder is the most accurate exactly order-invariant design (filled), at a quarter of the latency of the averaged joint encoder. Accuracy at the matched operating point (\cref{tab:matched}) against median latency per request on one idle A800 (280 cross-website requests, $K\le16$; the joint encoder in its faster setting). The independent encoder's cached point uses \cref{tab:reuse} ($K{=}16$, $r{=}1$).}
\label{fig:pareto}
\end{figure}

\Cref{fig:pareto} places the designs on accuracy, order invariance and latency. With identical inputs and only the model call timed, one joint request takes 168\,ms, the conditioned model 215\,ms and the independent model 190\,ms, or 57\,ms when its candidates are cached. The joint encoder is the fastest design and, with or without averaging, the most accurate, but it is order-dependent. Averaging five orders costs five requests (840\,ms), and two disjoint sets of five orders still disagree on 3.2\% of cross-website steps (two single orders: 6.0\%). The conditioned encoder gets exact invariance for one request at 1.3 times the joint latency, 0.5 and 1.3 points below the averaged joint encoder. The choice between \iris{}'s two encoders is set by the workload: when candidates exist before the request and recur, encode them once; when each request brings its own candidates, let them see the request.

\section{Related Work}
\label{sec:related}

\paragraph{Decision models and order dependence.} Language models change their multiple-choice answers under option reordering~\citep{pezeshkpour2024}, as do vision--language models~\citep{atabuzzaman2025mcqa}, and debiasing at inference time estimates and removes a prior over option identifiers or averages over permutations~\citep{pride}. Set-LLM changes the attention mask and positional encoding of a decoder so that a set of inputs is processed without order~\citep{setllm}, and Deep Sets~\citep{deepsets} and Set Transformer~\citep{settransformer} are permutation-invariant and equivariant set architectures. Among System One models, Runoff makes the decision order-invariant at inference time through a consideration set compared in every position~\citep{runoff}, and \jeva{} derives an architecture in which null awareness, order invariance, scalability and caching follow from encoding candidates independently~\citep{jeva2026}. MetaEncoder turns a multimodal decoder into a bi-encoder~\citep{metaencoder2026}. Scoring each option separately is order-independent~\citep{robinson2023}. Our work keeps \jeva{}'s decision module and asks what the candidate encoder may see when candidates are images.

\paragraph{Multimodal decisions over candidates.} \onejev{} answers typed questions about images and video with text candidates~\citep{onejev2026}. \citet{rerank2026} rerank the top images of a retriever by listing them in letter slots of a large vision--language model, and listwise rerankers of image--text documents likewise encode the candidates jointly~\citep{ranknexus}. Universal multimodal retrievers~\citep{uniir,vlm2vec2} embed queries and candidates separately and score them by similarity, which caches candidates but provides no null output and no interaction among candidates in the decision.

\paragraph{Encoders and prefix sharing.} Cross-encoders read query and candidate together and bi-encoders cache candidates~\citep{sbert}, as dense retrieval precomputes candidate vectors~\citep{dpr}; poly-encoders let cached candidate vectors attend to learned query codes~\citep{polyencoder}, and late interaction keeps per-token vectors and scores with MaxSim~\citep{colbert}. The state-conditioned encoder is a per-candidate cross-encoder whose request prefix is computed once, as in shared-prefix serving~\citep{sglang,hydragen}, combined with \jeva{}'s set decision module; our contribution is not this combination but the controlled decomposition of where joint encoding's advantage comes from. RetiJev copies such a cache to independent question branches on the same hybrid backbone but encodes a question's candidates jointly~\citep{retijev}; we fork per candidate. Set-Encoder lets passages exchange summary tokens at shared positions~\citep{setencoder}; PICASO averages state-space states over all context orders~\citep{picaso}; for LLM judges, missing cross-candidate interaction perturbs selections~\citep{liang2026kv}.

\paragraph{GUI agents.} Web and mobile agents select elements by ranking DOM candidates~\citep{mind2web}, by grounding a model's description on the page~\citep{seeact}, or by predicting locations directly from screenshots~\citep{seeclick,cogagent}; RankGround reranks screenshot crops before grounding~\citep{rankground}.

\section{Limitations}
\label{sec:limitations}
The vision tower is frozen (a LoRA on it did not help), and unpooled scoring is not tested. Isolation removes cross-candidate attention while candidates are encoded, not comparison at the readout, and tasks whose answer requires comparing candidates are not studied. The state-conditioned recipe and the joint encoder have three and five runs on both benchmarks; every other configuration is trained once. On AndroidControl, refitting the head recovers only 0.2--0.5 points of the gap left after seeing the request, and the rest is not decomposed. Null steps are few (4--6\% on Mind2Web), so null detection is measured coarsely; the quantile threshold of \cref{sec:null} assumes that the share of null requests in the calibration data matches deployment. The joint baseline is one model family and size. Reuse is timed at synthetic ratios. \iris{} starts from the released \jeva{}-4B checkpoint, whose curriculum data are not public; fine-tuning on visual candidates reduces text-only accuracy (\cref{app:jevbench}), and when candidates must be compared with a query image the independent encoder falls below its first-stage retriever (\cref{app:mmeb}).

\section{Conclusion}
\label{sec:conclusion}
A joint encoder's advantage over independent encoding comes mostly from candidates seeing the request, recovering 11--13 of 12--17 points on Mind2Web with exact order invariance; seeing each other during encoding adds 0.3 to 1.1 points, and abstention thresholds need post-hoc calibration.

\bibliography{references}
\bibliographystyle{plainnat}

\newpage
\appendix

\section{Implementation and Hyper-parameters}
\label{app:hparams}
\begin{table}[ht]\centering\small
\caption{Settings. Everything not listed is \jeva{}-4B's default.}
\label{tab:hparams}
\begin{tabular}{lp{0.74\linewidth}}
\toprule
Backbone & Qwen3.5-4B (vision tower and language model), bf16, SDPA attention \\
Vision tower & frozen; patch embedding computed as a matrix product (\cref{app:conv}) \\
LoRA & rank 64, $\alpha$ 128, dropout 0.05, on the attention and feed-forward projections of the language model; initialised from \jeva{}-4B \\
Pooling / projection & last token; two linear maps $4096\to768$, no bias \\
Decision stack & $N{=}1$ block, 8 heads, SwiGLU ratio 4; $\tau_{\text{set}}$ init 0.15, $\tau_{\text{dec}}$ 0.10, $\tau_{\text{cov}}$ 0.10 \\
Pixel budget & up to $200{,}704$ px per image ($\approx$256 tokens); Mind2Web crops fall below it \\
Optimiser & AdamW, lr $10^{-4}$ for the LoRA stages, $2\times10^{-5}$ for the abstention stage, no weight decay, gradient-norm clip 1.0 \\
Stage 1 & OneJev-Data, checkpoint at 80{,}000 rows of a planned epoch over the 92{,}707-row training split (the other 2{,}000 of the 94{,}707 released questions are held out for validation); 14.6\,h on $4\times$A800-80GB \\
Stage 2, independent & 7{,}751 Mind2Web steps, 2 epochs, 47\,min on $4\times$A800-80GB \\
Stage 2, conditioned & 7{,}751 Mind2Web steps, 2 epochs, one A800-80GB, 32 rows per optimiser step \\
Abstention stage & 5{,}005 rows (the 412 null steps $\times3$ and 3{,}769 of the 7{,}339 steps with a positive, 24.7\% null), 1 epoch, from the conditioned checkpoint with the best cross-task accuracy (epoch 2, step 50) \\
Readouts (\cref{sec:gap}) & set attention: 1 layer, 8 heads, zero-initialised outputs; token-level read: 8 heads, gated; multi-vector memory: 4 slots; each trained 1 epoch (token-level: stopped at 0.8) with new modules at $10$--$20\times$ the learning rate \\
\onejev{} fine-tuned & the 7{,}751 steps in \onejev{}'s crops-in-state format; LoRA $r{=}32$, $\alpha{=}64$ on the language model (vision tower frozen), lr $10^{-4}$, 2 epochs; 1.6\,h on one A800-80GB \\
AndroidControl & \iris{} independent and state-conditioned: 46{,}079 steps, 1 epoch from the stage-1 checkpoint, lr $10^{-4}$, 32 rows per optimiser step, one A800-80GB, final weights; \onejev{}: the Mind2Web recipe above with 1 epoch \\
\bottomrule
\end{tabular}
\end{table}

\section{Benchmark Construction Details}
\label{app:data}
Mind2Web crops are taken from the full-page screenshot with the element's bounding box, padded by 6\,px, grown symmetrically to a minimum side of 48\,px and downscaled to a longest side of 336\,px. The label is the tag followed by the first non-empty of \texttt{aria-label}, \texttt{alt}, \texttt{title}, \texttt{placeholder}, \texttt{name}, \texttt{value}, \texttt{type}, each truncated to 60 characters. Negatives are drawn from the step's negative candidates with a seed derived from the step id, so the build is reproducible and independent of the number of worker processes. Deduplication removed 109 candidates in 103 of 1{,}337 cross-task steps, 89 in 79 of 1{,}017 cross-website steps, 400 in 330 of 4{,}042 cross-domain steps and 714 in 603 of 7{,}751 training steps. Deduplication is applied to the input files before any model is run, and every result in the paper is inference on the deduplicated candidate sets. Dropping a duplicate's probability after inference and renormalising would not be equivalent, because the decision state attends over all candidates. For AndroidControl, the size and enabled filters apply to negatives only. Some targets, such as calendar entries and time-picker digits, are not marked clickable in the accessibility tree; when no clickable element contains the touch point, the smallest visible element with text or a content description that contains it becomes the gold. Touches that land only on the status bar are dropped as annotation errors, and the remaining steps without a gold are kept as null targets; most are taps on unlabelled, non-clickable elements such as search fields rendered inside web views, and some are taps on drawing canvases.

After shuffling, the gold position is at index 0 in 7.3 / 6.9 / 7.0\% of the cross-task / cross-website / cross-domain steps with a positive (6.25\% for a uniform position; records with fewer than 16 candidates raise it slightly).

\section{Additional Results}

\subsection{Accuracy against order dependence and null detection}
\label{app:tradeoff}
\begin{figure}[ht]
\centering
\includegraphics[width=\linewidth]{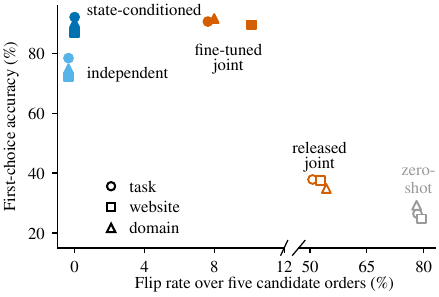}
\caption{Only joint encoding trades order invariance for accuracy: every \iris{} model has a flip rate of zero. One marker per model and split; \iris{} markers are offset horizontally for visibility. State-conditioned accuracies include the abstention margin of \cref{tab:main}; the x-axis is broken between 12 and 47\%. Grey: the unmodified Qwen3.5-4B chat model, zero-shot.}
\label{fig:tradeoff}
\end{figure}

\subsection{Zero-shot chat model}
\label{app:zeroshotchat}
The zero-shot row of \cref{tab:main} uses the released Qwen3.5-4B chat weights without any training. Each request lists the state, the question and the candidates as ``A: <crop> <tag>'', the crops as for \onejev{} (at most 448\,px on the long side), followed by one more letter for ``none of the above''. The model is run in non-thinking mode with the instruction to answer with a letter, and its scores are the logits of the first answer token restricted to the request's letters. Those letters receive 55--63\% of the probability of that token; most of the rest goes to the start of an explanation. Ranking accuracy on steps with a positive is 32.3 / 29.4 / 34.1\%, all-order accuracy 7.9 / 8.0 / 10.2\%, and at the matched operating point of \cref{tab:matched} it reaches 27.34 / 31.27\% on cross-website / cross-domain. It picks the first listed slot in 16--17\% of its answers, where a uniform choice over about 15 candidates would give 7\%, the selection bias reported for vision-language models on multiple-choice questions~\citep{atabuzzaman2025mcqa}.
\begin{table}[ht]
\centering\small
\setlength{\tabcolsep}{3pt}
\caption{Null detection on Mind2Web crops: AUROC of $p(\bot)$ and abstention rate on null steps (\%).}
\label{tab:null}
\begin{tabular}{lcccccc}
\toprule
& \multicolumn{3}{c}{AUROC} & \multicolumn{3}{c}{Abstention} \\
\cmidrule(lr){2-4}\cmidrule(lr){5-7}
Model & Task & Web & Dom. & Task & Web & Dom. \\
\midrule
Independent & 0.877 & 0.790 & 0.824 & 17.3 & 4.5 & 9.9 \\
Conditioned & 0.943 & 0.919 & 0.947 & 40.7 & 43.2 & 33.8 \\
\quad + abstention & \textbf{0.969} & \textbf{0.935} & \textbf{0.949} & \textbf{79.0} & \textbf{77.3} & \textbf{77.5} \\
\bottomrule
\end{tabular}
\end{table}

\subsection{Abstention margin}
\label{app:margin}
The released conditioned model predicts $\bot$ only when $p(\bot)\ge\max_i p_i+\delta$ (\cref{tab:margin}). The rule is applied identically in every order, so order invariance is unaffected.
\begin{table}[ht]\centering\small
\caption{Abstention margin for the conditioned model (\%), selected on cross-task.}
\label{tab:margin}
\begin{tabular}{lccc}
\toprule
$\delta$ & cross-task accuracy & accuracy on steps with a positive & abstention on null steps \\
\midrule
$-0.2$ & 91.02 & 91.64 & 81.5 \\
$0$ & 91.70 & 92.52 & 79.0 \\
$+0.1$ & 91.77 & 92.75 & 76.5 \\
$+0.24$ (selected) & \textbf{92.07} & 93.23 & 74.1 \\
$+0.4$ & 91.92 & 93.23 & 71.6 \\
\midrule
cross-website at $\delta{=}0$ / $+0.24$ & 86.53 / 87.02 & 86.9 / 87.5 & 77.3 / 77.3 \\
cross-domain at $\delta{=}0$ / $+0.24$ & 89.96 / 90.20 & 90.7 / 91.3 & 77.5 / 72.1 \\
\bottomrule
\end{tabular}
\end{table}

\subsection{What the abstention stage adds}
\label{app:abststage}
At each model's own threshold, the released model is 2.6 points above the stage-2 row of \cref{tab:main} on cross-website and 2.7 on cross-domain, but part of that difference is the position of the threshold, which drifts between checkpoints (\cref{app:gate}). The stage also changes two things at once: it oversamples null steps (24.7\% of its rows against 5.3\% in training) and trains one more epoch at a fifth of the learning rate. \Cref{tab:abststage} therefore compares, for three seeds of the gold-label recipe and at the matched operating point of \cref{tab:matched}, the stage-2 checkpoint selected on cross-task, the same checkpoint after the abstention stage, and a low-learning-rate control: the same checkpoint trained with the same learning rate, epoch and number of rows (5{,}005) on a seeded random sample of the ordinary training steps, with their natural 5.3\% null steps.

Before the results we fixed the rule that the null oversampling would be credited only if the abstention stage beat both the selected checkpoint and the control in every seed, by more than the seed-to-seed standard deviation of the recipe (about 0.8 points on cross-website). It does not. Over the three seeds the abstention stage adds $1.61\pm0.58$ points on cross-website and $0.94\pm0.92$ on cross-domain to the selected checkpoint, and the control adds $1.11\pm1.36$ and $0.52\pm0.61$. The stage beats the control by $-0.10$, $+1.57$ and $0.00$ points on cross-website and by $+0.02$, $+0.74$ and $+0.49$ on cross-domain; only seed 7 is significant per test ($p=0.044$ and $0.015$), and none after Holm's adjustment (\cref{app:cluster}). Null abstention and null AUROC rise by similar amounts with either. Most of what the stage adds is therefore one more epoch at a low learning rate; oversampling null steps contributes about half a point on average, inconsistently across seeds. We keep the released recipe as trained but do not attribute its gain to the oversampling.

\begin{table}[ht]\centering\small
\caption{The abstention stage against a low-learning-rate control at the matched operating point (quantile threshold $q=0.906$, \cref{app:gate}), gold-label recipe: accuracy on all steps (\%) of the selected stage-2 checkpoint, after the abstention stage and after the control; the difference between the stage and the control with its exact McNemar $p$; and null abstention and null AUROC (\%) in the same three columns' order.}
\label{tab:abststage}
\setlength{\tabcolsep}{3.5pt}
\begin{tabular}{llccccc}
\toprule
Seed & Split & Selected & + Abstention & + Control & Stage $-$ control ($p$) & Null abstention / AUROC \\
\midrule
42 & Website & 83.87 & 85.25 & 85.35 & $-0.10$ (1.000) & 70.5 / 79.5 / 81.8 \quad 91.6 / 93.5 / 94.4 \\
 & Domain & 89.16 & 89.04 & 89.02 & $+0.02$ (1.000) & 76.6 / 83.8 / 79.7 \quad 94.7 / 94.9 / 95.3 \\
7 & Website & 84.07 & 85.25 & 83.68 & $+1.57$ (0.044) & 75.0 / 79.5 / 77.3 \quad 91.6 / 93.6 / 92.1 \\
 & Domain & 88.57 & 89.98 & 89.24 & $+0.74$ (0.015) & 79.3 / 79.7 / 82.0 \quad 95.4 / 95.5 / 95.9 \\
13 & Website & 82.50 & 84.76 & 84.76 & $\phantom{+}0.00$ (1.000) & 81.8 / 84.1 / 81.8 \quad 92.6 / 94.2 / 93.3 \\
 & Domain & 87.46 & 88.99 & 88.50 & $+0.49$ (0.095) & 77.5 / 81.1 / 82.9 \quad 93.7 / 95.0 / 95.1 \\
\midrule
Mean & Website & 83.48 & 85.09 & 84.60 & $+0.49$ & \\
 & Domain & 88.40 & 89.34 & 88.92 & $+0.42$ & \\
\bottomrule
\end{tabular}
\end{table}

\subsection{The joint encoder with isolated candidates}
\label{app:isolation}
In \onejev{}'s crops-in-state prompt each candidate is a line ``$i$: crop label''. Isolating the candidates in a layer means that the tokens of line $i$ see the prefix (system prompt, task and ``Candidates:'') and their own tokens only, while the question, the option letters and the answer token after the list see every candidate; positions are unchanged. A mask cannot do this for Qwen3.5: 24 of its 32 layers are gated linear attention, a recurrence that carries line $i$ into line $i+1$. An isolated layer is therefore computed on the whole sequence and, as continuations of a fork of the prefix's cache, once per candidate line; candidate tokens take the second result and all other tokens the first. A layer's output at a position depends only on that layer's inputs at earlier positions, so the stitched sequence is exact. With nothing isolated, this layer loop reproduces the model's forward exactly, and perturbing one candidate leaves the final states of every other candidate unchanged. The gradients of the prefix-shared computation match a reference that recomputes the prefix for every line (fp32 cosine at least 0.998 in every layer tested).

Without retraining, isolating every layer of the fine-tuned joint encoder costs 30.9 points, also when the gold is listed first and its own states are unchanged (70.1 against 92.5\%), so most of that drop is the shift to an input the model never saw. Isolating only the eight full-attention layers, which is what a mask alone would do, costs 2.2 points. We therefore retrain \onejev{} from its release with every layer isolated and everything else as in \cref{tab:hparams} (data, LoRA, epochs, batch, seed; 216 steps, 4.1\,h on one A800-80GB), and evaluate it isolated. \Cref{tab:isolation} lists the results. The isolated model is one run. Against the mean of the five joint runs (\onejev{} seed sd 0.29 / 0.26 ranking points), the loss is 0.72 $[-1.98,+0.47]$ (website) and 0.28 $[-0.79,+0.22]$ (domain), each step's correctness averaged over the five runs; the original joint run is the best of the five on cross-website. As an exploratory equivalence check (margins not fixed in advance; the intervals reflect test sampling, not training seeds), the task-level 90\% bootstrap intervals of the loss against the original run are $[-2.22,-0.09]$ (website) and $[-0.89,+0.13]$ (domain), within $\pm2.5$ and $\pm1.5$ points but not $\pm2$ on cross-website; the \iris{} rows are the stage-2 models (seed 42 is the one of \cref{tab:modality}).

\begin{table}[ht]\centering\small
\caption{Ranking accuracy (\%) on Mind2Web steps with a positive, original order, with the paired difference from the fine-tuned joint encoder: exact McNemar $p$ and task-level 95\% bootstrap interval (2{,}000 resamples). One run per model; the five-run joint mean (compared in the text) comes from the labelling path of \cref{tab:main}, where the original run reads 92.09 / 94.27.}
\label{tab:isolation}
\setlength{\tabcolsep}{4pt}
\begin{tabular}{lcccc}
\toprule
& \multicolumn{2}{c}{Cross-website} & \multicolumn{2}{c}{Cross-domain} \\
\cmidrule(lr){2-3}\cmidrule(lr){4-5}
Model & Ranking & Difference ($p$; interval) & Ranking & Difference ($p$; interval) \\
\midrule
\onejev{} fine-tuned (joint) & 92.09 & -- & 94.24 & -- \\
\quad five runs, mean & 91.68 & -- & 94.13 & -- \\
\quad not retrained, all layers isolated & 61.15 & $-30.94$ ($<10^{-6}$; $[-35.0,-27.0]$) & -- & -- \\
\quad not retrained, full-attention only & 89.93 & $-2.16$ (0.001; $[-3.3,-0.9]$) & -- & -- \\
\quad retrained, all layers isolated & 90.96 & $-1.13$ (0.10; $[-2.4,+0.1]$) & 93.85 & $-0.39$ (0.21; $[-1.0,+0.3]$) \\
\iris{} conditioned, stage 2, seed 42 & 87.46 & $-4.63$ & 90.97 & $-3.27$ \\
\iris{} conditioned, stage 2, 3 seeds & 88.73 & $-3.36$ & 92.00 & $-2.24$ \\
\bottomrule
\end{tabular}
\end{table}

\subsection{Candidate interaction: details}
\label{app:encinter}
\paragraph{Readouts over pooled vectors.} Set attention also changes the mixture-trained independent model by $+0.9$ / $-0.5$ / $+1.6$ points (no continued-training control for that model). The token-level read reached 87.7 and 87.5\% on cross-website positives after 3{,}200 and 6{,}400 training rows, against 86.9\% before; the multi-vector memory was run only on a 300-step diagnostic subset. The set-attention block learned almost nothing (\cref{app:geometry}).

\paragraph{Interaction inside the encoder.} The gated delta rule~\citep{deltanet} is affine in its initial state, so a candidate started from $S_{\text{prefix}}+\gamma A_i$ produces exactly the output of one recurrence from the prefix plus that of a second recurrence with zero values started from $\gamma A_i$; both reuse the layer's queries, keys and gates, so the interaction costs one extra kernel call per Gated DeltaNet layer and no other computation. $A_i=\sum_{j\ne i}W_{ij}B_j$, where $B_j$ is candidate $j$'s write term (its final state when run from zero) and $W_{ij}$ is the mean, over every order in which the other candidates could precede $i$, of the decay applied to $B_j$ by the candidates after it (PICASO-S weights, \citealp{picaso}; exact for the scalar decay, the delta term ignored). $\gamma$ is a learned gate per layer and head, initialised at zero, so training starts from the plain model; padding steps are identity transitions. $A_i$ depends only on the multiset of the other candidates, so every layer is permutation-equivariant over a request's candidates and the decision stays order-invariant. Tests: the two-call decomposition equals one recurrence to $10^{-16}$ in fp64; the weights equal a brute-force average over all orders; with zero gates the vectors equal the plain model's; permuting candidates, replacing another request of the batch and changing the padded width leave them unchanged; the flip rate over five orders is zero on cross-website.

Trained for one epoch from the released model with the recipe of the continued-training control (seed 42), it changes ranking accuracy after head refit by $+0.10$ (website) and $-0.30$ (domain), $-0.22$ together ($[-0.73,+0.30]$), and is $+0.08$ on cross-task. The gates stay small (mean $|\gamma|$ 0.020), and at the trained weights the interaction moves the candidate vectors by 0.4\% and changes no first choice on 40 cross-website records. Two causes are visible. First, the aggregate is nearly the same for every candidate of a request: the part of $A_i$ specific to candidate $i$ is 11\% of the part common to all (a uniform leave-one-out mean gives $1/(K-1)=7\%$), and its weights depend on decay, not on content, so a candidate cannot single out the competitor it is confused with. Second, the training signal is sparse: the starting model ranks 97.4\% of the training steps correctly.

\subsection{What the gap to the isolated joint encoder is made of}
\label{app:gap}
The isolated \onejev{} of \cref{app:isolation} still leads the state-conditioned \iris{} (stage 2, seed 42) by 3.50 / 2.88 ranking points on cross-website / cross-domain. Three explanations remain: pooling each candidate to one vector loses information, the joint encoder's answer token compares candidates at the token level, or the two training recipes differ. We test them on frozen features: the \iris{} encoder is fixed, its candidate token states at layers 15 and 31 are exported, and heads are trained on them (3 seeds each; epoch chosen on cross-task). The checkpoint's ranking on the training split is 96.3\%, against 87--91\% on the test splits, so every probe is also fitted on 80\% of the cross-task tasks, which the encoder never saw, with the other 20\% choosing the epoch. Frozen probes are one-sided: a gain is evidence that pooling loses information, no gain does not rule out what end-to-end training could find.

\paragraph{The same head on both encoders.} One vector per candidate (the last token of its line, final layer) and one for the request are taken from the isolated \onejev{} and from \iris{}, and the same head (input RMS normalisation, projections of the same seed, \iris{}'s decision stack) is trained on each. \onejev{}'s vectors are 0.3--1.1 points worse on cross-website and 1.4--2.1 worse on cross-domain ($p<0.001$), while its own readout of all candidate tokens is 2.1--3.8 points above them. Its lead is thus not in a better pooled representation; the comparison is unfavourable to \onejev{}, whose line-final token was never trained to summarise a candidate.

\paragraph{Token-level readers on Iris's features.} Results are in \cref{tab:gap}. P0 continues training \iris{}'s head on its pooled vectors; P1 adds a two-block cross-attention reader over all of a candidate's tokens, P1-last the same reader over its last tokens only, P2 a second, gated stream over the other candidates' tokens, and T replaces the pooled vector by the reader's output (positive control). The pre-specified primary comparison, P1 $-$ P0 on cross-website and cross-domain together, is $+0.29$ points ($[+0.01,+0.57]$, task-level) when fitted on the training split and $+0.10$ ($[-0.18,+0.38]$) when fitted on cross-task: at most 19\% and 13\% of the gap. Reading the other candidates' tokens adds nothing ($+0.02$) or hurts ($-0.38$, Holm-adjusted $p=0.027$). T reaches P0, so the readers can learn.

\paragraph{The head.} P0 gains 2.4 / 1.8 ranking points over the checkpoint's head on this seed-42 model (1.4 / 1.6 when fitted on cross-task). Across checkpoints (\cref{tab:refit}; the refit is P0 on the training split, 3 seeds, epoch chosen on cross-task) the gain is large only where the head had converged worst: on seed 42 it is $+4.20$ points at the matched operating point ($p<0.001$), on seeds 7 and 13 $+0.49$ and $+1.21$, and on the released model nothing, its abstention stage having already trained the head at a low learning rate. The refit thus removes most of the seed-to-seed spread of the stage-2 final weights (matched cross-website accuracy 80.4--84.5\% before, 84.6--85.1\% after) and brings their mean to 84.89 / 87.99\%, against 85.09 / 89.34\% for the released recipe (\cref{tab:matched}); one more end-to-end epoch of the released model instead lowers its ranking accuracy. Against the three-seed stage-2 mean, the joint encoder's lead of 3.36 / 2.24 ranking points splits into 1.13 / 0.39 from candidates seeing each other (\cref{app:isolation}), 1.17 / 0.69 that the head refit recovers, and 1.06 / 1.16 that neither accounts for; frozen token readers do not reduce the last part. Against the mean of the four further \onejev{} runs (91.57 / 94.10 ranking), the lead is 2.84 / 2.10 points and the unexplained part 0.54 / 1.02. Skipping checkpoint selection costs little once the abstention stage follows: run from the stage-2 final weights of each seed, it reaches 84.66 / 88.84\% at the matched operating point, 0.43 / 0.49 below the released recipe ($[-1.31,+0.44]$ and $[-0.87,-0.13]$; seeds averaged per step), and from the refitted heads 84.46 / 88.76\%, so the refit adds nothing once the abstention stage has trained the head.

\begin{table}[ht]\centering\small
\caption{Refitting the head on the frozen encoder (mean of 3 refit seeds): ranking accuracy on steps with a positive and accuracy on all steps at the matched operating point of \cref{tab:matched} (\%), checkpoint $\to$ refit, cross-website / cross-domain. Both sides come from the batched feature export, whose bf16 arithmetic changes a few first choices, so checkpoint values differ from \cref{tab:main,tab:matched,tab:isolation} by up to 0.4 points; differences are taken within this path. The control is the released model trained one more epoch end to end.}
\label{tab:refit}
\setlength{\tabcolsep}{4pt}
\begin{tabular}{lcc}
\toprule
Checkpoint & Ranking & Matched operating point \\
\midrule
Stage 2, seed 42 & 87.36 $\to$ 89.72 / 90.94 $\to$ 92.74 & 80.43 $\to$ 84.63 / 86.74 $\to$ 88.43 \\
Stage 2, seed 7 & 89.00 $\to$ 89.55 / 92.98 $\to$ 93.09 & 84.46 $\to$ 84.96 / 88.89 $\to$ 88.86 \\
Stage 2, seed 13 & 89.83 $\to$ 90.44 / 92.09 $\to$ 92.23 & 83.87 $\to$ 85.09 / 86.27 $\to$ 86.69 \\
\quad mean & 88.73 $\to$ 89.90 / 92.00 $\to$ 92.69 & 82.92 $\to$ 84.89 / 87.30 $\to$ 87.99 \\
Released (seed 42) & 90.54 $\to$ 90.51 / 93.38 $\to$ 93.60 & 85.45 $\to$ 85.51 / 89.11 $\to$ 89.19 \\
\quad + one end-to-end epoch & 89.72 $\to$ 89.86 / 92.28 $\to$ 92.69 & 84.37 $\to$ 85.45 / 87.80 $\to$ 88.81 \\
\bottomrule
\end{tabular}
\end{table}

\paragraph{Where the errors are.} Splitting the steps with a positive of cross-website and cross-domain into quarters by the cosine between the gold's vector and its hardest negative's (continued-training control after head refit), ranking accuracy is 100.0, 99.8 and 99.1\% on the first three quarters and 69.6\% on the last (cosine above 0.9975): almost every error is a step whose two candidates are nearly the same vector. Nothing we tried helps there (\cref{tab:q4}): reading the other candidates' tokens on frozen features (P2 against P1, \cref{tab:gap}), the in-encoder interaction, training a LoRA on the vision tower, or one more end-to-end epoch, which lowers the released model's hard-quarter accuracy from 72.2 to 70.0\%.

\paragraph{Resolution.} \onejev{} reads crops with Qwen3.5's default minimum of $256^2$ pixels, \iris{} with $224^2$; since 93\% of the crops are smaller, the extra 25\% of image tokens come from upscaling. Retraining stage 2 at $256^2$ (seeds 42, 7 and 13, each paired with the stage-2 run of the same seed, each model's own head) changes ranking accuracy by $+0.17$ on cross-website ($[-0.70,+1.04]$) and $-0.38$ on cross-domain ($[-0.82,+0.08]$), and on the hard quarter by $+2.46$ and $-0.59$ (both intervals include zero). It is lower on cross-task ($-1.41$, $[-2.20,-0.60]$), mostly through seed 42, which is also 2.2 points lower on cross-domain while seed 13 is 1.2 higher. Resolution does not explain the gap.

\begin{table}[ht]\centering\small
\caption{Ranking accuracy (\%) on the hard quarter of steps (gold and hardest negative vectors most alike; defined on the released model for the vision and end-to-end rows, on the control for the interaction row, on the mean of the two models for resolution) and on the other steps, cross-website and cross-domain together unless stated; differences are paired, with task-level 95\% intervals.}
\label{tab:q4}
\setlength{\tabcolsep}{4pt}
\begin{tabular}{llcc}
\toprule
Comparison & Models & Hard quarter & Other steps \\
\midrule
Frozen features, \cref{tab:gap} (fit on train / on cross-task) & P2 $-$ P1 & $+0.25$ / $-1.17$ & -- \\
In-encoder interaction, after head refit & $-$ continued training & $-0.42$ & $-0.15$ \\
Vision-tower LoRA (one more epoch) & $-$ continued training & $-1.42$ $[-3.77,+0.97]$ & $-0.39$ \\
One more end-to-end epoch & released model $\to$ control & $72.23\to69.97$ & $99.55\to99.05$ \\
Resolution $256^2$, 3 seeds (website / domain) & $-$ stage 2 & $+2.46$ / $-0.59$ & -- \\
\bottomrule
\end{tabular}
\end{table}

\begin{table}[ht]\centering\small
\caption{Ranking accuracy (\%) on steps with a positive of heads trained on the frozen \iris{} stage-2 encoder (seed 42), mean of 3 seeds, fitted on the training split or on 80\% of the cross-task tasks. The checkpoint's own head reaches 87.36 / 90.94, the isolated \onejev{} 90.96 / 93.85.}
\label{tab:gap}
\begin{tabular}{lcccc}
\toprule
& \multicolumn{2}{c}{Fitted on training split} & \multicolumn{2}{c}{Fitted on cross-task} \\
\cmidrule(lr){2-3}\cmidrule(lr){4-5}
Head & Website & Domain & Website & Domain \\
\midrule
P0: decision head on pooled vectors & 89.72 & 92.74 & 88.73 & 92.51 \\
P0 with a 3-block head & 89.59 & 92.60 & 88.90 & 92.50 \\
P1: + reader over the candidate's tokens & 89.83 & 93.07 & 89.00 & 92.57 \\
P1-last: + reader over its last tokens only & 89.17 & 92.87 & 88.04 & 91.96 \\
P2: P1 + gated read of the other candidates & 90.03 & 93.05 & 88.52 & 92.21 \\
T: reader output only (positive control) & 89.83 & 92.77 & 88.73 & 91.81 \\
\bottomrule
\end{tabular}
\end{table}

\subsection{AndroidControl at a matched operating point}
\label{app:acmatched}
\Cref{tab:ac-matched} sets every AndroidControl model to the same false-abstention rate, 5\% on the positive steps of the synthetic-null validation split (368 steps with a positive and 47 null), the split on which the AndroidControl abstention margin is chosen. Unlike the quantile threshold on Mind2Web, this uses validation labels; no test split is used. Refitting the head on the frozen encoder (three refit seeds per model, epoch chosen on validation; batched export path) changes the state-conditioned model's ranking by $+0.47$ / $+0.16$ / $+0.33$ / $+0.15$ points (IDD / app / task / category, mean of the three seeds), against a gap of 2.5--4.2 to \onejev{}. Removing the gold from a step models a missing target, not a request that no candidate on the screen could satisfy, so these nulls measure only the first kind of null. Rows average five \onejev{} runs and three state-conditioned seeds (independent encoding: one run). The threshold transfers about equally for the two: on the test sub-splits, accuracy on steps with a positive falls 3.5--4.5 points below the threshold-free ranking accuracy for \onejev{} and 3.6--5.2 for the state-conditioned model (1.8--3.4 for independent encoding), so the matched-point gap to \onejev{} on the unmodified sub-splits, $-2.55$ / $-4.56$ / $-4.44$ / $-4.92$ points (IDD / app / task / category; episode-level 95\% intervals $[-3.66,-1.41]$, $[-5.74,-3.36]$, $[-5.68,-3.26]$, $[-6.32,-3.64]$, correctness averaged over runs per step), is close to the gap in original-order accuracy ($-2.83$ / $-3.89$ / $-3.98$ / $-4.32$) and in ranking ($-2.45$ / $-3.84$ / $-3.92$ / $-4.19$). Null AUROC is 83.1--83.9 for \onejev{}, 81.6--83.0 for the state-conditioned model and 73.0--75.2 for independent encoding.
\begin{table}[ht]\centering\small
\caption{AndroidControl at 5\% false abstention on the synthetic-null validation split, original order (\%). Unmodified: accuracy on all steps of the test sub-split. Synthetic null: accuracy on all steps / on steps with a positive / abstention on null steps of its synthetic-null copy; then ranking accuracy (null option ignored) and null AUROC, which do not depend on the threshold.}
\label{tab:ac-matched}
\resizebox{\textwidth}{!}{%
\begin{tabular}{llcccc}
\toprule
Model & Test copy & IDD & App & Task & Category \\
\midrule
\onejev{}-4B fine-tuned (joint), 5 runs & unmodified & 79.10 & 75.65 & 77.28 & 76.29 \\
 & synthetic null & 75.50 / 79.54 / 44.3 & 73.04 / 75.95 / 49.6 & 74.59 / 77.53 / 51.6 & 73.59 / 76.47 / 51.0 \\
 & ranking / AUROC & 83.08 / 83.8 & 80.46 / 83.1 & 81.56 / 83.9 & 80.61 / 83.4 \\
\iris{} independent & unmodified & 73.36 & 69.26 & 69.64 & 68.25 \\
 & synthetic null & 67.86 / 73.82 / 21.9 & 64.17 / 70.14 / 16.1 & 64.53 / 70.85 / 15.0 & 62.94 / 69.01 / 15.4 \\
 & ranking / AUROC & 77.23 / 73.0 & 72.03 / 74.7 & 72.64 / 75.2 & 70.76 / 75.0 \\
\iris{} state-conditioned, 3 seeds & unmodified & 76.55 & 71.09 & 72.84 & 71.37 \\
 & synthetic null & 72.99 / 76.82 / 43.4 & 68.82 / 71.36 / 48.4 & 70.26 / 72.95 / 49.2 & 68.91 / 71.50 / 48.6 \\
 & ranking / AUROC & 80.43 / 81.6 & 76.56 / 81.8 & 77.48 / 83.0 & 76.31 / 82.1 \\
\bottomrule
\end{tabular}}
\end{table}

\subsection{Task-level uncertainty}
\label{app:cluster}
The exact McNemar tests and bootstrap intervals elsewhere in the paper resample steps. Steps of one Mind2Web task share the website, the instruction and much of the page, so they are not independent: cross-website has 1{,}017 steps from 142 tasks and cross-domain 4{,}042 steps from 694. \Cref{tab:cluster} repeats the main paired comparisons with whole tasks as the resampling unit (2{,}000 resamples) and a task-level sign-flip permutation test (10{,}000 flips). In these rows, task-level intervals are up to 27\% wider than step-level ones, and no conclusion in the paper changes: the gap to the joint encoder at the matched operating point is not distinguishable from zero on cross-website and is just below zero on cross-domain; the abstention stage is significant on cross-domain for seeds 7 and 13 but not distinguishable from the low-learning-rate control; the distilled recipe's lower cross-domain null AUROC holds for both seeds.

The paper reports many paired tests, so we also adjust for multiplicity. We group the task-level tests into four families fixed by the questions they answer: the gap to the joint encoder at the matched operating point (mean over seeds), independent against conditioned encoding, the abstention stage (with and without, and against the low-learning-rate control), and distillation (\cref{app:distill}). Each split and metric of a family's comparisons is one test, 84 in all. Holm's adjustment within each family leaves below 0.05 all six encoder tests; for the abstention stage, its gain over the selected checkpoint on cross-domain accuracy for seeds 7 and 13 (adjusted $p=0.017$ and $0.020$), and for seed 13 on cross-domain null AUROC (0.017) and cross-website ranking (0.007), but none of its differences from the low-learning-rate control; and, for distillation, the cross-domain null AUROC drop of seed 7 (0.017), the drop from the distilled seed-42 abstention stage (0.017), and the stage-2 ranking gains of the distilled seed-42 run (0.030 and 0.007). Single tests quoted in \cref{app:distill} do not survive it: the 2.65-point cross-website loss from the distilled seed-7 abstention stage (0.19), the 0.94-point cross-domain loss of distilled seed 7 (0.53) and the seed-42 cross-domain AUROC drop (0.29). Their Benjamini--Hochberg $q$-values over all 84 tests are 0.029, 0.072 and 0.040. No test of the matched-point gap is significant after adjustment. Unless stated otherwise, $p$-values elsewhere in the paper are per test and unadjusted.
\begin{table}[ht]\centering\small
\caption{Paired differences $B-A$ (points) at the matched operating point ($q=0.906$; ranking and AUROC need no threshold), with step-level and task-level 95\% bootstrap intervals, the task-level $p$-value and its Holm adjustment within the family of the comparison. The first row averages each step's correctness over three \iris{} seeds, so seed-to-seed variation is not part of its error.}
\label{tab:cluster}
\resizebox{\textwidth}{!}{%
\begin{tabular}{llccccc}
\toprule
Comparison ($A\to B$) & Split, metric & $B-A$ & Step-level CI & Task-level CI & Task-level $p$ & Holm \\
\midrule
\onejev{} fine-tuned $\to$ \iris{} conditioned (3 seeds) & Website, all steps & $-0.95$ & $[-2.72,+0.82]$ & $[-3.01,+1.14]$ & 0.401 & 1.000 \\
 & Domain, all steps & $-0.87$ & $[-1.64,-0.12]$ & $[-1.77,-0.02]$ & 0.056 & 0.282 \\
 & Domain, ranking & $-0.83$ & $[-1.41,-0.21]$ & $[-1.48,-0.15]$ & 0.016 & 0.094 \\
\iris{} independent $\to$ conditioned (stage 2) & Website, all steps & $+11.60$ & $[+8.65,+14.36]$ & $[+8.10,+15.19]$ & $<0.001$ & 0.001 \\
 & Domain, all steps & $+13.14$ & $[+11.78,+14.55]$ & $[+11.40,+14.87]$ & $<0.001$ & 0.001 \\
Abstention stage, seed 42 & Website, all steps & $+1.38$ & $[-0.10,+3.05]$ & $[-0.63,+3.36]$ & 0.228 & 1.000 \\
 & Domain, all steps & $-0.12$ & $[-0.84,+0.59]$ & $[-0.88,+0.66]$ & 0.803 & 1.000 \\
Abstention stage, seed 7 & Website, all steps & $+1.18$ & $[-0.39,+2.76]$ & $[-0.58,+2.98]$ & 0.240 & 1.000 \\
 & Domain, all steps & $+1.41$ & $[+0.67,+2.15]$ & $[+0.70,+2.14]$ & $<0.001$ & 0.017 \\
Abstention stage, seed 13 & Domain, all steps & $+1.53$ & $[+0.74,+2.33]$ & $[+0.72,+2.38]$ & 0.001 & 0.020 \\
Low-LR control $\to$ stage, seed 7 & Website, all steps & $+1.57$ & $[+0.20,+3.05]$ & $[+0.20,+2.88]$ & 0.046 & 1.000 \\
Gold $\to$ distilled, final, seed 42 & Domain, null AUROC & $-1.89$ & $[-3.02,-0.80]$ & $[-3.17,-0.55]$ & 0.010 & 0.290 \\
Gold $\to$ distilled, final, seed 7 & Domain, null AUROC & $-2.30$ & $[-3.35,-1.27]$ & $[-3.53,-1.05]$ & $<0.001$ & 0.017 \\
\bottomrule
\end{tabular}}
\end{table}

\subsection{Training trajectory of the conditioned model}
\label{app:trajectory}
\Cref{tab:trajectory} lists the cross-task accuracy of the conditioned stage-2 checkpoints, and \cref{fig:trajectory} plots it.
\begin{table}[ht]\centering\small
\caption{Cross-task accuracy (\%) of the conditioned stage-2 checkpoints in the original order. From 1.2 epochs on, accuracy on steps with a positive stays at 90.8--92.8\%, and the total follows the abstention rate on the 81 null steps. The epoch-2, step-50 checkpoint scores 85.74\% on cross-website and 90.13\% on cross-domain.}
\label{tab:trajectory}
\begin{tabular}{lcccc}
\toprule
Checkpoint & Rows seen & Accuracy & Positive & Null \\
\midrule
epoch 1, step 100 & 3{,}200 & 86.69 & 91.5 & 12.3 \\
epoch 1, step 150 & 4{,}800 & 88.33 & 92.0 & 32.1 \\
epoch 1, step 200 & 6{,}400 & 87.51 & 91.4 & 27.2 \\
epoch 2, step 50 & 9{,}351 & \textbf{89.75} & 92.4 & 48.1 \\
epoch 2, step 100 & 10{,}951 & 89.68 & 91.2 & \textbf{65.4} \\
epoch 2, step 150 & 12{,}551 & 88.71 & 90.8 & 55.6 \\
epoch 2, step 200 & 14{,}151 & 88.11 & 92.2 & 24.7 \\
final (2 epochs) & 15{,}502 & 89.60 & \textbf{92.8} & 40.7 \\
\bottomrule
\end{tabular}
\end{table}
\begin{figure}[ht]\centering
\includegraphics[width=0.62\linewidth]{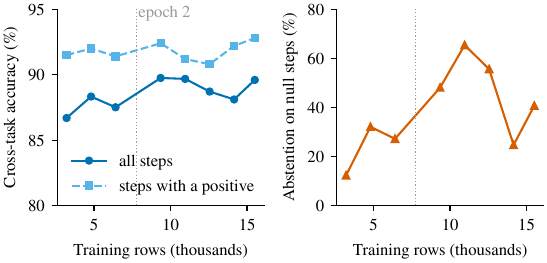}
\caption{The total follows abstention, not ranking: cross-task accuracy of the conditioned stage-2 checkpoints, in total, on steps with a positive, and abstention on the 81 null steps.}
\label{fig:trajectory}
\end{figure}

\subsection{Zero-shot transfer and stage 1}
\label{app:zeroshot}
With \jeva{}-4B's weights and no training on images, the independent model reaches 31.11 / 25.07 / 32.73\% on the three splits (chance 6.25\%), with a flip rate of zero on all 6{,}396 records: a decision head trained only on text vectors reads image-candidate vectors above chance. Stage 1 alone (80k OneJev-Data rows, no image candidates) gives 33.88 / 31.47 / 41.54\%. An earlier checkpoint of the same run, at 64k rows, gave 35.30 / 33.92 / 40.77\%, so the generalisation splits move by about 2 points between checkpoints that never saw Mind2Web. Stage 2 then adds 44 / 41 / 33 points: the decision head needs candidate-side training to read a new candidate modality.

\subsection{Mixture training}
\label{app:mixture}
A second independent model is trained from the stage-1 checkpoint for one epoch on 65{,}502 rows: the 7{,}751 Mind2Web steps twice, 40{,}000 reranking records mined from the M-BEIR training splits~\citep{uniir} of ten MMEB-V2 tasks~\citep{vlm2vec2} with the first-stage retriever's hard negatives, and 10{,}000 replayed OneJev-Data rows. On Mind2Web it reaches 78.91 / 74.04 / 75.06\%, matching the Mind2Web-only model.

\subsection{Typed-decision benchmarks}
\label{app:jevbench}
\Cref{tab:jevbench} reports \iris{} on the typed-decision benchmarks of \jeva{} and \onejev{}, whose candidates are text.
\begin{table}[ht]\centering\small
\caption{Accuracy (\%) on JevBench~\citep{jevbench} and the 154 labelled dev rows of ImajevBench~\citep{imajev}. The flip rate of \iris{} is zero on every split. Training on visual candidates alone reduces accuracy on the text-only JevBench rows; replaying OneJev-Data recovers part of the loss on the multimodal-state rows it resembles.}
\label{tab:jevbench}
\resizebox{\textwidth}{!}{%
\begin{tabular}{lcccc}
\toprule
Model & JevBench orig. & JevBench easy & JevBench hard & ImajevBench dev \\
\midrule
OpenJev-27B (joint; reported in \citet{metaencoder2026}) & 98.61 & 100.0 & 76.58 & -- \\
MetaEncoder-30B (bi-encoder; reported in \citet{metaencoder2026}) & 94.44 & 100.0 & 73.87 & 85.55 \\
\iris{}-4B, \jeva{}-4B weights & \textbf{100.0} & \textbf{100.0} & 60.36 & 70.13 \\
\iris{}-4B stage 1 & 95.83 & \textbf{100.0} & 48.65 & 74.68 \\
\iris{}-4B stage 2 (Mind2Web crops only) & 75.00 & \textbf{100.0} & 29.73 & 61.04 \\
\iris{}-4B mixture & 73.61 & 97.92 & 44.14 & 65.58 \\
\bottomrule
\end{tabular}}
\end{table}

\subsection{Reranking retrieval candidates}
\label{app:mmeb}
Following \citet{rerank2026}, we rerank the top 16 candidates of Qwen3-VL-Embedding-2B on the 12 image-retrieval tasks of MMEB-V2~\citep{vlm2vec2} (1{,}000 queries per task, hit@1). Our reproduction of the first stage gives a mean hit@1 of 74.88 (\cref{tab:mmeb}).
\begin{table}[ht]\centering\small
\caption{Reranking the top 16 of the first-stage retriever on MMEB-V2 image retrieval (hit@1). Reported rows are from \citet{rerank2026}. VisDial and Wiki-SS-NQ are held out from the mixture.}
\label{tab:mmeb}
\setlength{\tabcolsep}{3pt}
\resizebox{\textwidth}{!}{%
\begin{tabular}{lcccccccccccc|c}
\toprule
& COCO i2t & VisNews i2t & VisDial & COCO t2i & VisNews t2i & WebQA & EDIS & Wiki-SS & CIRR & NIGHTS & OVEN & FashionIQ & mean \\
\midrule
First stage & 74.6 & 75.4 & 81.5 & 78.0 & 73.1 & 90.1 & 92.6 & 83.7 & 65.1 & 69.2 & 76.5 & 38.8 & 74.9 \\
Reranker-8B (supervised; reported) & 77.7 & 79.6 & 89.5 & 82.2 & 75.1 & 92.7 & 92.4 & 86.4 & 74.1 & 70.8 & 78.6 & 41.2 & 78.4 \\
27B letter slots, zero-shot (reported) & 70.5 & 74.2 & 84.8 & 80.0 & 76.2 & 92.4 & 95.0 & 87.3 & 72.0 & 68.8 & 78.0 & 41.9 & 76.8 \\
\iris{} stage 1, zero-shot & 41.4 & 40.2 & 15.4 & 21.6 & 27.2 & 34.4 & 42.1 & 13.5 & 4.4 & 4.9 & 41.6 & 5.5 & 24.4 \\
\iris{} mixture & 47.6 & 49.4 & 30.4 & 44.8 & 47.0 & 82.0 & 69.3 & 48.8 & 4.5 & 6.0 & 49.3 & 4.6 & 40.3 \\
\bottomrule
\end{tabular}}
\end{table}
The independent model does not reach its first stage on any task. On the three tasks whose query is an image or an image with an edit (CIRR, NIGHTS, FashionIQ) it stays at 4--6\%, at or below a uniform guess over the 16 candidates (6.1, 6.2 and 5.1\%, given that the first stage places the gold in its top 16 for 97.7, 99.5 and 81.7\% of queries): a candidate vector computed without the query image cannot carry a comparison with it. This is the comparison the state-conditioned encoder supplies, and reranking, where each query brings its own candidates, is a setting in which giving up cross-request reuse costs nothing; we have not yet trained a conditioned reranker.

\subsection{Crop or label: which part of a candidate is used}
\label{app:modality}
Every crop candidate carries a short text label (\cref{app:data}), so accuracy on crop candidates alone does not show that the model reads the crop. We compare three candidate inputs: the crop and its label, as everywhere else in the paper; the crop alone, with an empty label; and the label alone, with no image. The state, the question, the gold index and the null steps are unchanged. Candidates that become identical once one part is removed, such as two elements labelled \texttt{div} once their crops are gone, are kept, so that a model unable to tell them apart gives them equal scores and the ablation counts that inability.

The stage-2 state-conditioned and independent models on Mind2Web (\cref{tab:main}, before the abstention stage) and the state-conditioned and independent models on AndroidControl (\cref{tab:ac}) are evaluated on copies of their test splits in which every candidate keeps only its crop or only its label. These models were trained with both parts and are tested with one missing, so the drop combines the loss of information with a shift between training and test inputs; we report these numbers as indicative only.

All models are evaluated in the original order. \Cref{tab:modality} reports ranking accuracy on steps with a positive, which does not depend on the abstention threshold, and \cref{tab:modality-auroc} the Mind2Web null AUROC.

\begin{table}[ht]\centering\small
\caption{Crop or label, without retraining: ranking accuracy (\%) on steps with a positive with full candidates (crop and label), the crop only and the label only. Every difference from full candidates has an exact McNemar $p<10^{-6}$. Mind2Web null AUROC is in \cref{tab:modality-auroc}.}
\label{tab:modality}
\resizebox{\textwidth}{!}{%
\begin{tabular}{llccccccc}
\toprule
& & \multicolumn{3}{c}{Mind2Web ranking} & \multicolumn{4}{c}{AndroidControl ranking} \\
\cmidrule(lr){3-5}\cmidrule(lr){6-9}
Encoder & Candidates & Task$^\dagger$ & Website & Domain & IDD & App & Task & Category \\
\midrule
State-conditioned & crop + label & 92.99 & 87.46 & 90.97 & 80.05 & 78.00 & 78.60 & 77.93 \\
 & crop only & 80.41 & 81.29 & 82.96 & 75.15 & 74.03 & 74.94 & 73.99 \\
 & label only & 61.70 & 51.08 & 54.35 & 58.05 & 55.23 & 57.48 & 55.67 \\
Independent & crop + label & 82.80 & 75.44 & 78.82 & 76.91 & 71.51 & 72.28 & 70.43 \\
 & crop only & 67.91 & 68.24 & 68.46 & 71.46 & 67.85 & 68.87 & 66.96 \\
 & label only & 57.17 & 43.17 & 48.35 & 55.07 & 49.92 & 51.93 & 49.88 \\
\bottomrule
\end{tabular}}
\end{table}

\begin{table}[ht]\centering\small
\caption{Crop or label, without retraining: null AUROC (\%) on Mind2Web. AndroidControl has too few natural null steps.}
\label{tab:modality-auroc}
\begin{tabular}{llccc}
\toprule
Encoder & Candidates & Task$^\dagger$ & Website & Domain \\
\midrule
State-conditioned & crop + label & 94.3 & 91.9 & 94.7 \\
 & crop only & 92.2 & 88.4 & 87.2 \\
 & label only & 77.6 & 80.7 & 74.8 \\
Independent & crop + label & 87.7 & 79.0 & 82.4 \\
 & crop only & 83.4 & 74.6 & 76.6 \\
 & label only & 78.5 & 73.2 & 72.9 \\
\bottomrule
\end{tabular}
\end{table}

Removing the crops costs far more than removing the labels: 26 to 37 points of Mind2Web ranking accuracy against 6 to 15, and 20 to 23 points on AndroidControl against 3 to 5. The models rely mainly on the crops, and the label adds to them. Because the models were trained with both parts, these drops combine the lost information with the shift between training and test inputs.

\paragraph{With retraining.} To remove the shift, the state-conditioned stage-2 recipe (from the same stage-1 checkpoint, 2 epochs, learning rate $10^{-4}$, seed 42, final weights, no checkpoint selection) is retrained on crop-only and on label-only copies of the 7{,}751 Mind2Web training steps and evaluated on the matching test copies (\cref{tab:modality-retrained}). Without labels the model loses 2.5 to 5.3 points of ranking accuracy; without crops it loses 29 to 35, and retraining recovers only 1.3 to 2.5 of them, because many labels are shared or uninformative (\texttt{div}, \texttt{a}). The model reads the crops, and the label is a small addition. Each variant is trained once.

\begin{table}[ht]\centering\small
\caption{Crop or label, retrained: the state-conditioned stage-2 recipe trained and tested on the same candidate input (Mind2Web, seed 42, final weights). Ranking accuracy on steps with a positive / null AUROC (\%); in parentheses, the ranking difference from full candidates, all with exact McNemar $p<0.05$.}
\label{tab:modality-retrained}
\begin{tabular}{lccc}
\toprule
Candidates & Task$^\dagger$ & Website & Domain \\
\midrule
crop + label & 92.99 / 94.3 & 87.46 / 91.9 & 90.97 / 94.7 \\
crop only & 87.74 / 94.2 ($-5.25$) & 84.99 / 89.0 ($-2.47$) & 87.09 / 89.8 ($-3.87$) \\
label only & 64.17 / 79.5 ($-28.82$) & 52.83 / 77.2 ($-34.64$) & 55.68 / 76.8 ($-35.29$) \\
\bottomrule
\end{tabular}
\end{table}

\subsection{Patch embedding as a matrix product}
\label{app:conv}
On an A800 in bf16 without \texttt{flash\_attn}, the patch-embedding \texttt{Conv3d} of the Qwen3.5 vision tower took 5.2\,s for one $448\times448$ image. The convolution has kernel equal to stride, so on the flattened patch tensor it equals a matrix product with the reshaped kernel. Replacing the module's forward by this product, with weights unchanged and a maximum absolute difference of $2\times10^{-6}$, brings the pass to 15\,ms. All latencies in this paper use this replacement, for every model alike; the slowdown is specific to cuDNN's bf16 Conv3d without \texttt{flash\_attn} on this setup, and with a fast native kernel the replacement is unnecessary.

\section{Calibrating the Abstention Gate}
\label{app:gate}

This appendix records how the instability of the abstention decision was found, traced to the gate's formula, and removed. All analyses use saved predictions and the checkpoints' head parameters; no model is retrained. The scripts, \texttt{diagnose\_abstention.py} and \texttt{test\_gate\_*.py}, are in the code release.

\subsection{Symptom}
Repeating the gold-label recipe and the distilled recipe of \cref{app:distill} with a second seed changed cross-website accuracy by up to 1.4 points, more than the distillation effect under test. Splitting accuracy into its two components on every periodic checkpoint showed where the variation lies: ranking accuracy on steps with a positive (null option ignored) moves within 92--93\% during epoch 2, while abstention on the 81 null steps moves between 27 and 94\% (\cref{tab:gate-symptom}), and over-abstention also costs steps with a positive (81.5\% total accuracy at 93.8\% abstention).

\begin{table}[ht]\centering\small
\caption{Cross-task accuracy (\%) of periodic stage-2 checkpoints. Ranking is stable; abstention is not.}
\label{tab:gate-symptom}
\begin{tabular}{lccc}
\toprule
Checkpoint (gold labels, seed 7) & Accuracy, all steps & Ranking accuracy, steps with a positive & Abstention on null steps \\
\midrule
epoch 1, step 100 & 81.53 & 91.88 & 93.8 \\
epoch 1, step 200 & 86.61 & 91.00 & 27.2 \\
epoch 2, step 50 & 89.30 & 92.60 & 49.4 \\
epoch 2, step 150 & 89.60 & 92.52 & 80.2 \\
epoch 2, step 200 & 90.73 & 93.23 & 58.0 \\
\bottomrule
\end{tabular}
\end{table}

\subsection{The defect in the gate}
The gate computes $p(\bot)=\sigma\big(-(\max_i s_i-b)/\tau_{\text{cov}}\big)$ with $s_i=\langle z_N,\hat u_i\rangle$. Three facts combine. (i) The choice distribution $q=\mathrm{softmax}(s/\tau_{\text{dec}})$ is invariant to $s_i\mapsto s_i+c$, so the cross-entropy on $q$ has zero gradient along $c$: the overall score level is a flat direction of the choice loss. (ii) The candidate vectors of a request share a dominant direction $\mu$ that carries 98--99.8\% of their squared norm (measured on the stage-2 models), so $s_i\approx\langle z_N,\mu\rangle+\langle z_N,\hat u_i-\mu\rangle$ and the first term is exactly such a common level, moved as a side effect of every update to the state. Only the coverage loss constrains it: its gradient along a common shift $c$ is $(\rho-v)/\tau_{\text{cov}}$ per step ($v=1$ with a positive, 0 for a null step), so positive steps push the level up until $\rho$ saturates near 1 and the 412 null steps (5.3\% of training) push it down. (iii) The gate compares $\max_i s_i$ with a fixed $b$ at temperature $\tau_{\text{cov}}\approx0.14$, so a level shift of 0.1 moves the coverage logit by about 0.7. With a constant learning rate, the level is consistent with a weakly restored random walk, and the abstention rate follows it.

The checkpoints confirm each part. Recovering $\max_i s_i=b+\tau_{\text{cov}}\operatorname{logit}(1-p(\bot))$ from the predictions and the head parameters, across the 45 saved checkpoints of five runs (gold labels with seeds 42, 7, 13; distilled with seeds 42, 7), $b$ stays between 0.0061 and 0.0064 and $\tau_{\text{cov}}$ between 0.140 and 0.143, whereas the median best-candidate score on null steps ranges from $-0.41$ to $+0.74$. \Cref{tab:gate-r2} shows how much of the variation of $\operatorname{logit}$(abstention) each candidate cause explains. The level of the scores on null steps explains 90--92\%; the separation between null and positive steps, and the null AUROC, explain little. Within one run, the spread of abstention across epoch-2 checkpoints (sd 15.0 points for the gold-label runs, 11.0 for the distilled ones) exceeds the spread between the final weights of different seeds (7.7 and 6.1), so seed-to-seed differences in abstention are mostly the point at which a drifting level is frozen.

\begin{table}[ht]\centering\small
\caption{$R^2$ of $\operatorname{logit}$(abstention on cross-task null steps) with each candidate cause, over stage-2 checkpoints.}
\label{tab:gate-r2}
\begin{tabular}{lcccc}
\toprule
Runs & Score level on null steps & Null--positive score gap & Null AUROC & $\tau_{\text{cov}}$ \\
\midrule
Gold labels (25 checkpoints) & 0.92 & 0.15 & 0.06 & 0.00 \\
Distilled recipe (17 checkpoints) & 0.90 & 0.30 & 0.21 & 0.15 \\
\bottomrule
\end{tabular}
\end{table}

\subsection{Candidate remedies}
Every candidate score is recoverable from a saved prediction: $s_i=\max_j s_j+\tau_{\text{dec}}(\log q_i-\log\max_j q_j)$. We therefore compared gate statistics offline, each with one threshold shared by all checkpoints (pooled 5\% false abstention on cross-task steps with a positive), by the spread of the null abstention rate across the 24 epoch-2 and final stage-2 checkpoints.

\emph{Removing the level.} The shift-invariant statistic $\tau_{\text{dec}}\,\mathrm{KL}(U\,\|\,q)=\tau_{\text{dec}}\log\frac1K\sum_i e^{s_i/\tau_{\text{dec}}}-\frac1K\sum_i s_i$, where $U$ is uniform over the $K$ candidates, keeps most of the separation (AUROC 94.4 against 95.7) but makes the spread worse (sd 37.3 against 16.5 points): it is not invariant to the scale of the scores, which also drifts. \emph{Removing level and scale per request}, $(\max_i s_i-\bar s)/\mathrm{sd}_i(s_i)$, is stable but loses the separation (AUROC 78.6): by Samuelson's inequality it is bounded by $\sqrt{K-1}$, so a clear positive cannot stand out. \emph{Removing level and scale across requests}, $(\max_i s_i-\mu)/\bar\sigma$ with $\mu$ the median candidate mean and $\bar\sigma$ the median candidate standard deviation of unlabelled calibration requests, is a positive affine map of $\max_i s_i$ within a checkpoint and so keeps AUROC exactly.

A first comparison computed $\mu$ and $\bar\sigma$ on a random half of the records and evaluated on the other half, but drew the half anew for every checkpoint, so different checkpoints were scored on different null steps; even the uncalibrated gate's spread rose by 2.3 points under that protocol. \Cref{tab:gate-remedies} uses record-level 2-fold cross-fitting instead (folds fixed by record id, statistics for each fold from the other fold's unlabelled records, all 81 null steps scored) and a paired bootstrap over records (1{,}000 resamples, the same records for every checkpoint and gate). The across-request normalisation lowers the spread significantly. A two-level variant that additionally standardises each checkpoint's statistic by the mean and standard deviation of its central 80\% does not: aligning the bulk of the positive distribution does not align its 5\% tail, where the threshold sits.

\begin{table}[ht]\centering\small
\caption{Gate statistics on 25 cross-task stage-2 checkpoints with 2-fold cross-fitting: null AUROC (\%) and the standard deviation (points) of null abstention across checkpoints.}
\label{tab:gate-remedies}
\begin{tabular}{lccccc}
\toprule
Gate statistic & AUROC & Abstention sd & 95\% CI & Difference to current & 95\% CI \\
\midrule
Current: $\operatorname{logit} p(\bot)$, fixed threshold & 95.69 & 16.5 & [13.8, 19.5] & -- & -- \\
$(\max_i s_i-\mu)/\bar\sigma$ & 95.69 & 9.2 & [7.2, 12.7] & $-7.3$ & [$-8.9$, $-5.2$] \\
Median/MAD level, no alignment & 95.68 & 15.3 & [13.2, 18.4] & $-1.1$ & [$-3.4$, $+1.6$] \\
Median/MAD level, central-80\% alignment & 95.67 & 17.0 & [15.1, 20.0] & $+0.6$ & [$-1.2$, $+3.1$] \\
\midrule
Quantile threshold on $\operatorname{logit} p(\bot)$ ($q=0.906$) & 95.69 & 4.9 & & & \\
Per-checkpoint threshold set with labels & 95.69 & 5.4 & & & \\
\bottomrule
\end{tabular}
\end{table}

\subsection{The quantile threshold}
The labelled per-checkpoint threshold (each checkpoint at exactly 5\% false abstention) leaves a spread of 5.4 points, the variation in how well checkpoints separate null steps; the rest of the 16.5 is threshold drift. A label-free version of that threshold is a quantile: the model abstains when $\operatorname{logit} p(\bot)$ exceeds the $q$-quantile of $\operatorname{logit} p(\bot)$ over unlabelled calibration requests. Being a threshold on the model's own statistic, it leaves AUROC unchanged; with cross-fitting it reaches a spread of 4.9 points and a false-abstention spread of 0.3 points.

\emph{Choosing $q$.} The fraction of requests on which the model abstains is $1-q=(1-\pi)f+\pi r$, where $\pi$ is the share of null requests, $f$ the false-abstention rate and $r$ the recall on null steps, so $q$ fixes an operating point given the null share. We choose one $q$ for all models on cross-task, the split already used to select checkpoints and margins: bisection on $q$ until pooled false abstention over the 25 checkpoints is 5\% gives $q=0.906$ (cross-task: $\pi=6.1\%$, $r\approx79\%$, and $0.939\cdot5\%+0.061\cdot79\%\approx9.4\%=1-q$). A common $q$ fixes the abstention budget, not the false-abstention rate: at $q=0.906$ the cross-task false abstention is 5.0--5.2\% for the joint runs, 4.2--5.3\% for the state-conditioned models and 6.9\% for independent encoding, depending on their null recall. $q$ itself was chosen with cross-task labels; only the per-model thresholds are label-free. The cross-task false-abstention rate of 5\% is therefore true by construction; the evidence is the spread and the transfer below. The spread does not depend on the exact choice (\cref{tab:gate-q}): $q$ trades false abstention against recall, while the spread stays at 3--6 points. Without labels, $q$ can instead be set by an abstention budget.

\begin{table}[ht]\centering\small
\caption{The quantile level sets the operating point, not the stability (25 cross-task checkpoints, cross-fitted).}
\label{tab:gate-q}
\begin{tabular}{lcccccc}
\toprule
$q$ & 0.85 & 0.88 & 0.90 & 0.906 & 0.92 & 0.94 \\
\midrule
Abstention on null steps, mean (\%) & 88.9 & 85.6 & 80.8 & 79.1 & 74.4 & 63.6 \\
Abstention, sd across checkpoints & 3.3 & 3.8 & 4.7 & 4.9 & 5.3 & 4.3 \\
False abstention, mean (\%) & 10.3 & 7.4 & 5.6 & 5.0 & 3.8 & 2.3 \\
\bottomrule
\end{tabular}
\end{table}

\emph{Transfer.} \Cref{tab:gate-transfer} applies the gates to eleven final weights (stage-2 and post-abstention models of the gold-label recipe with seeds 42, 7 and 13 and of the distilled recipe of \cref{app:distill} with seeds 42 and 7, and the distilled model whose abstention stage starts from the gold-label selection) on cross-website and cross-domain, with every constant taken from cross-task. The quantile threshold cuts the spread of abstention by 61--65\%. Its false abstention on cross-website is 9.1\% rather than 5\%, because cross-website has fewer null steps (4.3\%) than cross-task (6.1\%) and harder positives: the quantile carries the calibration data's null share, which is the assumption to check in deployment. Computing the quantile on the target split itself (cross-fitted) brings it to 6.5\%.

\begin{table}[ht]\centering\small
\caption{Eleven final weights on the other two splits (\%): mean and standard deviation across models.}
\label{tab:gate-transfer}
\begin{tabular}{llcc}
\toprule
Split & Gate (constants from cross-task) & Abstention mean / sd & False abstention mean / sd \\
\midrule
Cross-website & Fixed threshold on $\operatorname{logit} p(\bot)$ & 79.1 / 8.4 & 9.7 / 4.5 \\
& Quantile threshold & 79.5 / 2.9 & 9.1 / 1.1 \\
Cross-domain & Fixed threshold on $\operatorname{logit} p(\bot)$ & 73.4 / 11.1 & 5.6 / 2.4 \\
& Quantile threshold & 76.4 / 4.3 & 5.6 / 0.8 \\
\bottomrule
\end{tabular}
\end{table}

The analysis also changes how recipes should be compared. Total accuracy mixes ranking with the position of a drifting threshold; ranking accuracy with the null option ignored and null AUROC do not depend on the threshold and vary several times less across seeds, while accuracy on steps with a positive does depend on it.

\section{Distillation from the Joint Encoder}
\label{app:distill}
Since the fine-tuned joint encoder is more accurate, we tried distilling it into the state-conditioned model. The teacher's option distribution on each training step, averaged over five candidate orders, is renormalised without its null slot, softened with temperature $T=2$ and mixed with the gold label in equal parts; the null target stays the gold label. Everything else follows the gold-label recipe, so the two recipes differ only in their targets. \Cref{tab:distill} compares them with the quantile threshold of \cref{app:gate}. Distillation gives no consistent gain: on cross-website it changes accuracy by $+0.79$ ($p=0.42$, exact McNemar) and $-0.69$ ($p=0.51$) for the two seeds, and on cross-domain by $-0.32$ ($p=0.45$) and $-0.94$ ($p=0.015$). It lowers the cross-domain null AUROC in both seeds ($-1.9$ and $-2.3$ points, 95\% bootstrap intervals $[-3.1,-0.8]$ and $[-3.3,-1.3]$); the drop is absent after stage 2 ($+0.3$ and $+0.8$, intervals including zero) and arises in the abstention stage, which for the distilled recipe changes accuracy by $+0.20$ ($p=0.90$) and $-2.65$ ($p=0.001$) on cross-website and by $-0.20$ and $-0.12$ on cross-domain, and lowers the cross-domain null AUROC by 1.7 and 1.3 points (intervals $[-2.5,-0.9]$ and $[-2.2,-0.4]$). These $p$-values and intervals are per test and resample steps; with tasks as the unit and Holm's adjustment within the distillation family, the cross-domain AUROC drops of the seed-7 final model and of the distilled seed-42 abstention stage remain significant and the other single tests above do not (\cref{app:cluster}). With the uncalibrated gate, the first seed showed a significant gain of 2.65 points on cross-website ($p=0.002$) that the second seed did not reproduce; it is an instance of the threshold drift of \cref{app:gate} combined with a low-scoring gold-label seed.

\begin{table}[ht]\centering\small
\caption{Distilled against gold-label targets for the state-conditioned recipe (final models, after the abstention stage).}
\label{tab:distill}
\begin{tabular}{lcccc}
\toprule
& \multicolumn{2}{c}{Accuracy, quantile threshold (\%)} & \multicolumn{2}{c}{Null AUROC (\%)} \\
\cmidrule(lr){2-3}\cmidrule(lr){4-5}
Recipe (seeds 42 / 7) & Website & Domain & Website & Domain \\
\midrule
Gold labels & 85.25 / 85.25 & 89.04 / 89.98 & 93.5 / 93.6 & 94.9 / 95.5 \\
Distilled & 86.04 / 84.56 & 88.72 / 89.04 & 94.3 / 94.6 & 93.0 / 93.2 \\
\bottomrule
\end{tabular}
\end{table}

\end{document}